\documentclass[letterpaper]{article} 
\usepackage{aaai24}  
\usepackage{times}  
\usepackage{helvet}  
\usepackage{courier}  
\usepackage[hyphens]{url}  
\usepackage{graphicx} 
\usepackage{natbib}  
\usepackage{caption} 
\usepackage{algorithm}
\usepackage{algorithmic}

\usepackage{newfloat}
\usepackage{listings}
\DeclareCaptionStyle{ruled}{labelfont=normalfont,labelsep=colon,strut=off} 
\floatstyle{ruled}
\newfloat{listing}{tb}{lst}{}
\floatname{listing}{Listing}
\usepackage{xcolor}
\usepackage{colortbl}
\usepackage{adjustbox}
\usepackage{multirow}
\usepackage{makecell}
\usepackage{amsmath}
\usepackage{amssymb}
\usepackage{caption}
\usepackage{amsthm}
\usepackage{enumitem}
\usepackage{graphicx}
\usepackage{subcaption}
\usepackage{appendix}

\theoremstyle{plain}
\newtheorem{thm}{Theorem}
\newtheorem{lemma}{Lemma}

\theoremstyle{definition}
\newtheorem{defn}{Definition}
\newtheorem{prop}{Proposition}
\newtheorem{coro}{Corollary}
\definecolor{PG}{RGB}{20, 160, 40}
\definecolor{Purple}{RGB}{140, 30, 180}
\definecolor{TableYellow}{RGB}{255, 255, 230}
\definecolor{TableGray}{RGB}{230, 230, 230}

\newcommand{\gD}{\mathcal{D}}
\newcommand{\gC}{\mathcal{C}}
\newcommand{\gX}{\mathcal{X}}
\newcommand{\gI}{\mathcal{I}}
\newcommand{\gL}{\mathcal{L}}
\newcommand{\gM}{\mathcal{M}}
\newcommand{\gMp}{\mathcal{M}_\pi}
\newcommand{\qf}[2]{q(c_{#1}=c_{#2}|x_{#1},x_{#2};f)}

\newcommand{\qfm}[2]{q(c_{#1}=c_{#2}|x_{#1},x_{#2}\in\gMp;f)}
\newcommand{\qfmm}[2]{q(c_{#1}=c_{#2}|x_{#1},x_{#2}\in\gM;f)}
\newcommand{\pf}[2]{p(c_{#1}=c_{#2}|x_{#1},x_{#2})}
\newcommand{\qfo}[2]{q(c_{#1}=c_{#2}|x_{#1},x_{#2};f^*)}
\newcommand{\argmax}[1]{\underset{#1}{\arg\max}\;}
\newcommand{\argmin}[1]{\underset{#1}{\arg\min}\;}

\newlength{\torchindent}
\definecolor{commentcolor}{RGB}{110,154,155}
\newcommand{\pseudocode}[1]{\ttfamily\bfseries{#1}}
\newcommand{\STATEPTCOMMENT}[2]{\STATE{\hskip #1\torchindent}\textcolor{commentcolor}{\pseudocode{\# #2}}}
\newcommand{\INLINEPTCOMMENT}[1]{\hfill\textcolor{commentcolor}{\pseudocode{\:\# #1}}}
\newcommand{\STATEPTCODE}[2]{\STATE{\hskip #1\torchindent}\textcolor{black}{\pseudocode{#2}}}
\newcommand{\STATEPTBREAK}{\STATEPTCODE{0}{}}

\title{DUEL: Duplicate Elimination on Active Memory\\for Self-Supervised Class-Imbalanced Learning}
\author{
    Won-Seok Choi\textsuperscript{\rm 1},
    Hyundo Lee\textsuperscript{\rm 1},
    Dong-Sig Han\textsuperscript{\rm 1},
    Junseok Park\textsuperscript{\rm 1},
    Heeyeon Koo\textsuperscript{\rm 2},\\
    Byoung-Tak Zhang\textsuperscript{\rm 1,3,}\footnote{Corresponding author.}
}
\affiliations{
    \textsuperscript{\rm 1}Seoul National University\\
    \textsuperscript{\rm 2}Yonsei University\\
    \textsuperscript{\rm 3}AI Institute of Seoul National University (AIIS)\\
    \{wchoi, hdlee, dshan, jspark, btzhang\}@bi.snu.ac.kr
}

\begin{document}

\maketitle

\begin{abstract}
Recent machine learning algorithms have been developed using well-curated datasets, which often require substantial cost and resources. On the other hand, the direct use of raw data often leads to overfitting towards frequently occurring class information. To address class imbalances cost-efficiently, we propose an active data filtering process during self-supervised pre-training in our novel framework, Duplicate Elimination (DUEL). This framework integrates an active memory inspired by human working memory and introduces distinctiveness information, which measures the diversity of the data in the memory, to optimize both the feature extractor and the memory. The DUEL policy, which replaces the most duplicated data with new samples, aims to enhance the distinctiveness information in the memory and thereby mitigate class imbalances. We validate the effectiveness of the DUEL framework in class-imbalanced environments, demonstrating its robustness and providing reliable results in downstream tasks. We also analyze the role of the DUEL policy in the training process through various metrics and visualizations.
\end{abstract}

\section{Introduction}

Recent machine learning algorithms are heavily influenced by the quantity and quality of data. However, when agents collect data in real-world environments, the class distribution of the unprocessed data is long-tailed, indicating that data from certain classes are acquired much more frequently than others~\citep{liu2019large}. When trained on such raw data without any processing, deep learning models tend to overfit to these \textit{frequent} classes. Therefore, adaptive data refinement during the training process is essential to mitigate class imbalances cost-efficiently. Traditional methods have been developed based on resampling~\citep{buda2018systematic,pouyanfar2018dynamic} and reweighting~\citep{cao2019learning,cui2019class,tan2020equalization} techniques. However, these methods require class information for each data point, which increases the cost of preprocessing and labeling. Even semi-supervised approaches~\citep{wei2021crest,kim2020distribution} still require a fine-tuned support set that can reflect the target data distribution. To address these issues, recent research~\citep{yang2020rethinking,liu2021self} has proposed self-supervised pretraining techniques that can be trained with minimally processed data and demonstrate improved performance in class-imbalanced environments.

Self-supervised learning~\citep{chen2020simple,he2020momentum,zbontar2021barlow}, which is the modernized form of metric learning~\citep{khosla2020supervised,sohn2016improved}, has the advantage of acquiring positive samples via augmentation which reflects the inductive bias in the data. Similarly, by using data in either a mini-batch or a memory as negative samples, SSL methods do not require class information to collect such negatives. 
However, in a class-imbalanced environment, we have observed that the relationships within and between classes are unevenly reflected when training an SSL model. This imbalance leads to performance degradation in both mini-batch and memory-based methods. Based on this empirical evidence, we claim that an active memory that can alleviate the imbalances between latent classes is necessary for self-supervised class-imbalanced learning.

In this context, we mimic a well-known human cognitive process to achieve an active memory. Human working memory~\citep{baddeley1999working,baddeley2012working} is a prominent cognitive concept that explains how humans deal with extreme class-imbalances via an active data filtering process. Figure \ref{fig1}.A shows the mechanism of human working memory. The Central Executive System (CES), which is a supervisory subsystem of working memory, \textit{inhibits dominant information}~\citep{miyake2000unity,wongupparaj2015relation} from perceived data and remembers it while maximizing the amount of information~\citep{baddeley1999working}. These cognitive phenomena support our hypothesis that eliminating the most duplicated data will increase the distinctiveness information within memory.

\begin{figure*}[tb]
    \begin{center}
        \includegraphics[width=0.8\textwidth]{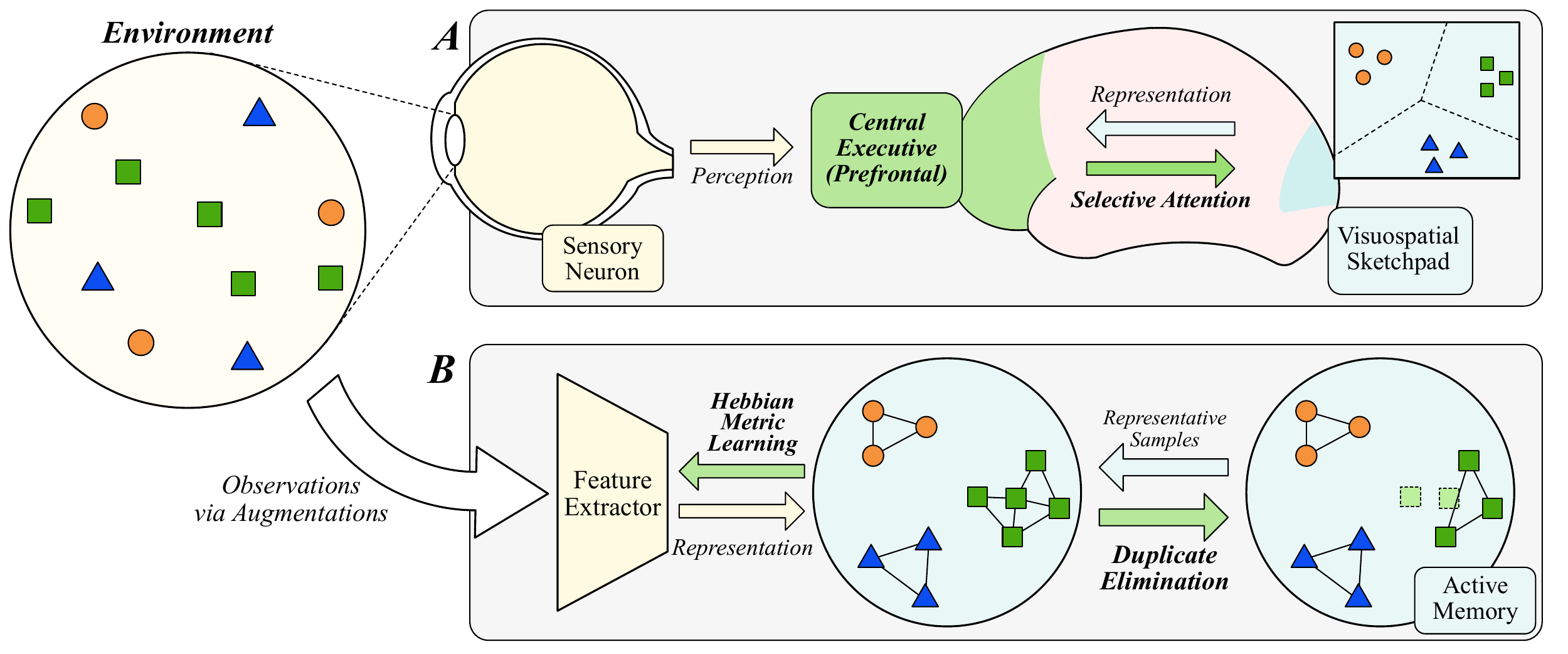}
        \caption{Visualizations of the concepts of working memory and our proposed DUEL framework. (A) Real-world agent \textit{perceives} data from the environment and maps the representation to solve the task. Working memory finds semantically duplicated signals and reduces them to maximizes the total amount of information. (B) Inspired by this cognitive process, we design the Duplicate Elimination (DUEL) framework. With mutual duplication probability, the representations form a graph structure (center) and are filtered out (right) to gradually maximize the distinctiveness information.}
        \label{fig1}
    \end{center}
\end{figure*}

To compute the amount of the information, we define \textit{distinctiveness} information, which measures how different a data point is from other data points. We also introduce Hebbian Metric Learning (HML) which directly optimizes \textit{distinctiveness} information while reducing information among co-fired similar data inspired by the characteristics of Hebbian learning~\citep{hebb2005organization,lowel1992selection}.
We show that to generalize HML in class-imbalanced environments, an active memory is essential. In this case, a policy for managing the memory should maximize the distinctiveness information in the memory.

As an implementation of our method, we propose the Duplicate Elimination (DUEL) framework, a novel SSL framework tailored for class-imbalanced environments.
Figure \ref{fig1}.B provides a conceptual visualization of the proposed DUEL framework.
The framework consists of two components: an active memory and a feature extractor. The feature extractor is trained with both the current data and the additional data from the active memory, while the active memory eliminates the most redundant data with the extracted representations. By iteratively updating both components, the DUEL framework can extract robust representations even in highly class-imbalanced environments.

This study provides the following contributions:
\begin{itemize}
  \item \textbf{Memory-integrated Hebbian Metric Learning.} We define HML which optimizes the information-based metric between the data from a Hebbian perspective. We show that an active memory that maximizes distinctiveness information is essential to extend HML to class-imbalanced environments.
  \item \textbf{Memory Management Policy.} Inspired by working memory, we design a memory management policy that eliminates the most duplicated element in the memory.
  \item \textbf{DUEL Framework.} We propose the DUEL framework for self-supervised class-imbalanced learning. To simulate class-imbalanced environments, we assume that one \textit{dominant} class occurs more frequently than others. In class-imbalanced environments, performance degradation has been observed with conventional self-supervised learning methods. On the other hand, even with the dramatically class-imbalanced data, the DUEL framework maintains stable performance in downstream tasks. We also validate the DUEL framework with more realistic environments with long-tailed class distributions and observe consistent results.
\end{itemize}

\section{Revisiting Metric Learning from a Hebbian-based Perspective}

In this section, we discuss Hebbian Metric Learning (HML), which allows us to represent the optimization problem of both the feature extractor and the memory from the same perspective. HML consists of Hebbian information and distinctiveness information terms, which aim to make representations of data with the same latent class similar while maximizing the diversity of information.

\subsection{Problem Definition}

The data distribution $\gD$ is a joint distribution of the observation $x\in\gX$ and its corresponding \textit{latent class} $c\in\gC$~\citep{saunshi2019theoretical,ash2021investigating,awasthi2022more}. Data with each latent class $c$ has a distinct data distribution $\gD_c$, and latent classes constitute the class distribution $c\sim\rho$. In this case, the joint distribution $p(x, c)$ can be expressed as follows:
\begin{equation*}
p(x,c):=\rho(c)\cdot \gD_{c}(x).
\end{equation*}

Our objective is to fit an estimated distribution $q(x,c;f)$ with a feature extractor $f:\gX \rightarrow \mathcal{Z}$ to the true distribution $p(x,c)$ by minizing the Kullback-Leibler divergence $D_{KL}(p(x,c)||q(x,c;f))$. 
However, since latent class is not directly accessible, indirect methods such as metric learning is needed to compute $q(x,c;f)$.

\subsection{Hebbian Metric Learning}

To represent $p(x,c)$ and $q(x,c;f)$ without directly using latent class, we define the probability that two data samples share the same latent class as mutual duplication probability. In this case, $\qf{i}{j}$ can be computed through a similarity metric of the representations of two data samples in the latent space.

\begin{defn}[Mutual duplication probability]
Let $(x_i,c_i)$, $(x_j,c_j)\sim\gD$. The mutual duplication probability $\qf{i}{j}$ with feature extractor $f$ is defined as follows:
\begin{equation}
\label{eqn-collision}
\qf{i}{j}:=\text{sim}^*(f(x_i),f(x_j)).
\end{equation}
\end{defn}
$\text{sim}^*$ denotes an arbitrary metric function which satisfies the property as the probability: the range should be bounded to $[0,1]$. With mutual duplication probability, the duplication density functions $P$ and $Q$ are derived via Bayes' theorem.
\begin{align*}
P(x_{i},x_{j}):=&\:p(x_i,x_j|c_{i}=c_{j})\\=&\:\frac{\pf{i}{j}p(x_i)p(x_j)}{\mathbb{E}_{x_k\sim \gD} \left[ \pf{i}{k}\right]}
\end{align*}
\begin{align*}
Q(x_{i},x_{j};f):=&\:q(x_i,x_j|c_{i}=c_{j};f)\\=&\:\frac{\qf{i}{j}p(x_i)p(x_j)}{\mathbb{E}_{x_k\sim \gD} \left[ \qf{i}{k}\right]}
\end{align*}

$P$ and $Q$ represent normalized joint distributions of two data samples sharing the same class. Through the use of \textit{message passing}~\citep{pearl1988probabilistic}, we can represent the two joint distributions, $p(x,c)$ and $q(x,c;f)$, using their density functions $P$ and $Q$. In Proposition \ref{prop_hml}, we show that minimizing $D_{\text{KL}}(p(x,c)||q(x,c;f))$ is equivalent to minimizing $D_{\text{KL}}(P||Q)$ and it becomes Hebbian Metric Learning.

\begin{prop}[Hebbian Metric Learning] Minimizing $D_{\text{KL}}(p(x,c)||q(x,c;f))$ is equivalent to minimizing $\gL_{\text{HML}}(f;\gD)$, which can be derived as:
\label{prop_hml}
\begin{align*}
\argmin{f}D_{\text{KL}}(p||q) &=\argmin{f}D_{\text{KL}}(P||Q) \\
&=\argmin{f}\underbrace{\left(\gI_h(f;\gD)-\gI_d(f;\gD)\right)}_{\gL_{\text{HML}}(f;\gD)}
\end{align*}
where $\gI_h(f;\gD)$ and $\gI_d(f;\gD)$ are denoted as \textit{Hebbian} information and \textit{Distinctiveness} information respectively.
\begin{align}
\label{hebbian}
\gI_h(f;\gD):=&\:\mathbb{E}_{x_i\sim \gD} \left[ I_h(x_i;f,\gD)\right] \\
\gI_d(f;\gD):=&\:\mathbb{E}_{x_i\sim \gD} \left[\gI_d(x_i;f,\gD)\right]
\label{dupinfo}
\end{align}
\begin{align*}
I_h(x_i;f,\gD)&:=\mathbb{E}_{x_j\sim \gD_i^+} \left[ -\log \qf{i}{j} \right]\\
\gI_d(x_i;f,\gD)&:=-\log \left( \mathbb{E}_{x_j\sim \gD} \left[ \qf{i}{j}\right]\right)
\end{align*}
\end{prop}
$\gD_i^+$ represents the distribution of data that belong to the same latent class of $x_i$. For each data $x_i$, Hebbian information $\small I_h(x_i;f,\gD)$ is defined as mean information of mutual duplication probability with positive samples from $\gD_i^+$. In Hebbian learning~\citep{hebb2005organization,lowel1992selection}, the learning process strengthens the connections between similar data, which means that the Hebbian information between the two data should be minimized.

On the other hand, for each data $x_i$, distinctiveness information $\gI_d(x_i;f,\gD)$ is estimated information of the proportion of class $c_i$ from the distribution $\gD$. The expected value of distinctiveness information $\gI_d(x_i;f,\gD)$ becomes a measurement indicating how diversely the latent class information is distributed within the data distribution. For agents, it is essential to acquire as much information as possible from the observations in the given environment to form diverse representations. This property, known as the distinctiveness effect~\citep{parker1998restorff, waddill1998distinctiveness}, becomes crucial in extracting the richest representation from the data.

The optimization for $\gL_{\text{HML}}$ can be interpreted as finding $f^*$ that minimizes Hebbian information among similar data, while preventing collapsed representation by maximizing $\gI_d$ as a regularization term. Figure \ref{figConceptHML} visualizes the concept of Hebbian Metric Learning. 

\begin{figure}[t]
    \begin{center}
        \includegraphics[width=0.8\columnwidth]{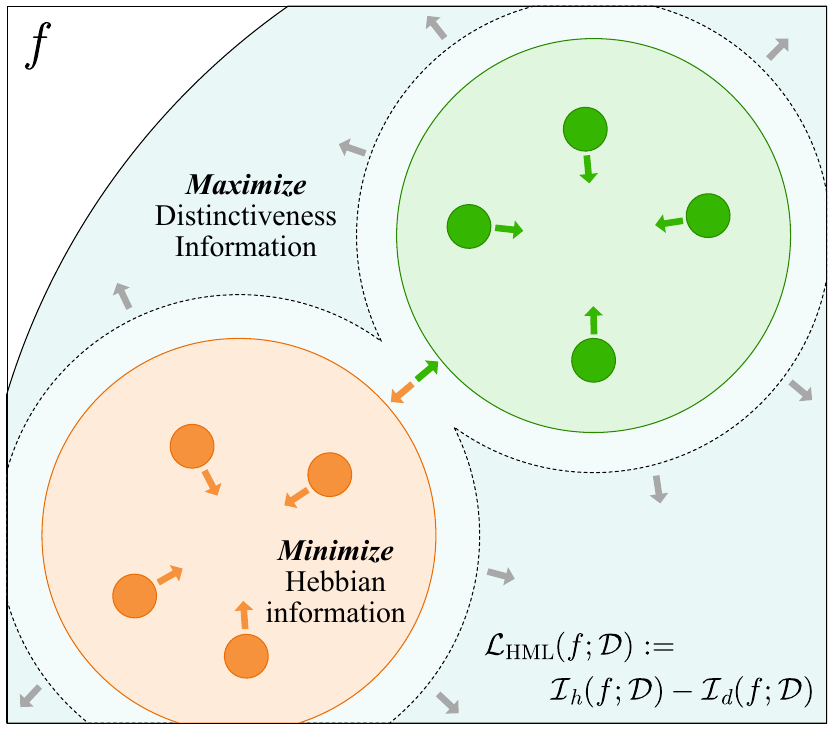}
        \caption{Conceptual Visualization of Hebbian Metric Learning. HML minimizes the Hebbian information while maximizing the distinctiveness information.}
        \label{figConceptHML}
    \end{center}
\end{figure}

\section{Memory-integrated HML\\for Class-imbalanced Environment}
\label{sec:active_data_filtering}
Conventional metric learning frameworks~\citep{khosla2020supervised,sohn2016improved,caron2020unsupervised,chen2020simple,chen2020improved,he2020momentum} often assume that $\gD$ is an \textit{oracle} with evenly distributed class information: $\forall c\in\gC, \rho_{c}=1/|\gC|$.
However, when the accessible data is unrefined, data with some dominant classes may occur more frequently than others, resulting in class-imbalances which can hinder the formation of robust representations in SSL. To deal with these imbalances, maintaining a memory which stores the data selectively can be a breakthrough. Thus we extend Hebbian Metric Learning with a memory for the empirical distribution $\gD'$ with class distribution $\rho$ by introducing a memory $\gM$.

\begin{prop}[HML Bound] Let $\gD$ and $\gD'$ be the oracle and the empirical data distribution, respectively. Then the upper bound of ideal HML loss is formulated as below:
\label{hml_bound}
\begin{align}
&\small \gL_{\text{HML}}(f;\gD) \le \nonumber\\&\underbrace{\lambda {\cdot} \gI_{h}(f;\mathcal{D'})-\gI_d(f;\gD',\gM)+|\gI_d(f;\gD',\gM)-\gI_d(f;\gD)|}_{\gL_{\text{M-HML}}(f,\gM;\gD')}
\label{emp_hml_ineq}
\end{align}
where $\gI_d(f;\gD',\gM):=\mathbb{E}_{x_i\sim \gD'} \left[\gI_d(x_i;f,\gM)\right]$ denotes the empirical distinctiveness information in the memory $\gM$ and $\lambda=1/(|\gC|\cdot\rho_{\min})$.
\end{prop}

The proof is provided in Appendix A. We denote the upper bound in Equation \ref{emp_hml_ineq} as $\gL_{\text{M-HML}}(f,\gM;\gD')$, and adopt it as an objective function to minimize. In Theorem \ref{optimality}, we show that the optimal feature extractor $f^*$ that minimizes $\gL_{\text{M-HML}}$ with optimal memory $\gM^*$ is also optimal for $\gL_{\text{HML}}$ with oracle data distribution.
\begin{thm}[Optimality of M-HML]
Assume that an optimal memory $\gM^*\simeq \gD$ exists. With the memory $\gM^*$, the ideal loss  $\gL_{\text{HML}}$ and empirical loss $\gL_{\text{M-HML}}$ shares the optimal feature extractor $f^*$:
\begin{equation*}
\gL_{\text{M-HML}}(f^*,\gM^*;\gD')=\gL_{\text{HML}}(f^*;\gD)=-\log |\gC|
\end{equation*}
where mutual duplication probability with $f^*$ satisfies the following property:
\begin{equation*}
\qfo{i}{j}=\begin{cases}1\qquad c_i=c_j\\0\qquad c_i\neq c_j.\end{cases}
\end{equation*}
\label{optimality}
\end{thm}
\paragraph{Proof Sketch.}{We show that the bound of difference between two losses becomes zero with optimal feature extractor $f^*$ and memory $\gM^*$. The bound contains two terms: $|\lambda\cdot\gI_{h}(f;\gD')-\gI_{h}(f;\gD)|$ and $|\gI_d(f;\gD',\gM)-\gI_d(f;\gD)|$.}

We find $f$ which satisfies $|\lambda\cdot\gI_{h}(f;\gD')-\gI_{h}(f;\gD)|=0$ by setting $\qf{i}{j}=1$ for $c_i=c_j$. We show that $|\gI_d(f^*;\gD',\gM^*)-\gI_d(f^*;\gD)|=0$ with optimal feature extrator $f^*$. Then $|\gL_{\text{M-HML}}(f^*,\gM;\gD')-\gL_{\text{HML}}(f^*;\gD)|=0$ is satisfied and $\gL_{\text{HML}}(f^*;\gD)=-\log|C|$.

By utilizing Theorem \ref{optimality} and Proposition \ref{prop_hml}, the optimal feature extractor $f^*$ which minimizes $D_{\text{KL}}(p||q)$ is also optimal for $\gL_{\text{M-HML}}(f^*,\gM^*;\gD')$ with the memory $\gM^*$. In the next section, we describe a procedural methodology to effectively optimize $\gL_{\text{M-HML}}$ by adopting an active memory based on distinctiveness information $\gI_{d}$.

\begin{figure*}[tb]
    \begin{center}
        \includegraphics[width=0.9\textwidth]{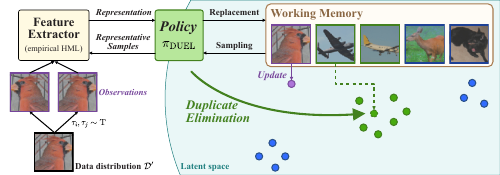}
        \caption{Visualization of general DUEL framework. Our method stores various data for the negative samples by Duplicate Elimination. The DUEL policy selects the most duplicated sample in memory (green) and replaces it with current data (purple).}
        \label{figDUEL}
    \end{center}
\end{figure*}

\section{Duplicate Elimination on Active Memory\\with Hebbian Metric Learning}
 
\subsection{Memory Management Policy}

In the previous section, we propose the objective function $\gL_{\text{M-HML}}$ for both the memory and feature extractor. Since memory is a finite set which stores limited number of the incoming data, optimizing $\gM$ is a discrete process of deciding which data to store. Therefore, we split the objective function $\gL_{\text{M-HML}}$ into objective function of memory while fixing the feature extractor, and objective function of feature extractor while fixing the memory. In this case, we set the memory $\gM$ to hold $K$ representative data points: $\gM\in\gX^K$.
\begin{align}
f^*:=&\,\argmin{f}\left(\lambda \cdot \gI_{h}(f;\mathcal{D'})-\gI_d(f;\gD',\gM)\right)
\label{emp_hml}\\
\small \gM^*:=&\,\argmin{\gM\in\gX^K}|\gI_d(f;\gD',\gM)-\gI_d(f;\gD)| \nonumber\\
=&\,\argmax{\gM\in\gX^K}\gI_d(f;\gD',\gM)
\label{mem_optimize}
\end{align}
To remove $\gI_d(f;\gD)$ term with oracle distribution $\gD$ in Equation \ref{mem_optimize}, we assume that $\gI_d(f;\gD)\ge\gI_d(f;\gD',\gM)$ and simplify Equation \ref{mem_optimize} as $\arg\max_{\gM\in\gX^K}\gI_d(f;\gD',\gM)$, which implies maximizing the distinctiveness information in the memory.
To optimize the memory, we introduce an active memory $\gMp$ that is procedurally updated by a memory management policy $\pi$. Since finding an optimal policy $\pi^*$ is NP-hard, we design its approximated policy inspired by the human cognitive process.

\subsection{Duplicate Elimination Policy on Active Memory}
\label{duel_policy}

Working memory is an active memory connected to human sensory-motor neurons, allowing humans to selectively concentrate on necessary information from the environment to achieve their goals. In order to mimic the behavior of CES, which is inhibiting the dominant information~\citep{miyake2000unity,wongupparaj2015relation}, we design a memory management policy $\pi_{\text{DUEL}}$ based on distinctiveness information.

\begin{defn}[Duplicate Elimination] Let the new data $x_{\text{new}}$ be provided to active memory $\gMp$. 
The DUEL policy $\pi_{\text{DUEL}}$ is a policy which chooses the $J$-th element $x_J\in\gMp$ with the minimum value of $\gI_d(x_j;f,\gMp)$.
\label{def-DUEL}
\begin{equation}
J=\argmin{j\in\{1..K\}}\gI_d(x_j;f,\gMp)
\label{criterion}
\end{equation}
\end{defn}

$\pi_{\text{DUEL}}$ in Definition \ref{def-DUEL} replaces the element with the least distinctiveness information, which is the most duplicated element in the memory. The replacement is carried out gradually, one element at a time. We show that the process of $\pi_{\text{DUEL}}$ is \textit{safe} in the sense that it increases total amount of information as in Appendix A. Figure \ref{figDUEL} illustrates the behavior of $\pi_{\text{DUEL}}$. $\pi_{\text{DUEL}}$ finds the densest area (green) of the latent space and ejects the most duplicated element (dotted outline). The plural region (blue) is not influenced by this replacement and leaving this region intact will increase distinctiveness information in the memory.

\begin{algorithm}[t]
    \caption{DUEL Framework with the policy $\pi_{\text{DUEL}}$}
    \label{alg-DUEL-framework}
    \textbf{Model : } feature extractor $f_\theta$, memory $\mathcal{M}$ \\
    \textbf{Input : } empirical data distribution $\gD'$, batch size $B$, memory size $K$, learning rate $\eta$ \\
    \textbf{Output :} trained feature extractor $f_{\theta^*}$
    
    \begin{algorithmic}[1]
    \STATE $\theta\leftarrow \theta_{0}$
    \STATE $\mathcal{M}\leftarrow \mathcal{M}_0$
    \WHILE {$\theta$ is not converged}
    \STATE $\{(x_{b},x_{b}^+)\}_{b=1}^B \leftarrow \text{Sample}(\gD')$
    \STATE Compute $\gL_{\text{InfoNCE}}(\{(x_{b},x_{b}^+)\}_{b=1}^B,\mathcal{M};f_\theta)$
    \STATE $\theta\leftarrow \theta - \eta\nabla_{\theta}\gL_{\text{InfoNCE}}$
    \FOR {$b \in \{1,\cdots,B\}$} 
    \STATE $\mathcal{M} \leftarrow \mathcal{M}\cup\{x_{b}\}$
    \STATE $J \leftarrow \argmin{j\in\{1..(K+1)\}}\gI_d(x_j;f_{\theta},\gM)$ \hfill $\triangleright$\:$\pi_{\text{DUEL}}$
    \STATE $\mathcal{M}\leftarrow \mathcal{M}\setminus\{x_{J}\}$
    \ENDFOR
    \ENDWHILE
    \end{algorithmic}
\end{algorithm}

Since reducing memory usage and time comsumption of $\pi_{\text{DUEL}}$ is crucial to implement our model, we optimize the policy $\pi_{\text{DUEL}}$ to minimize the time consumption with an affordable amount of additional resources. Details of the implementation and analyses on the resource usage are in Appendix D. With $\pi_{\text{DUEL}}$, we now propose the Duplicate Elimination (DUEL) framework.

\subsection{DUEL Framework}
\label{DUEL_framework}

The DUEL framework is an self-supervised learning framework which optimizes $\gL_{\text{M-HML}}$ in Equation \ref{emp_hml_ineq}. The procedure of our framework is shown in Figure \ref{figDUEL}. To sample positive samples from $\gD_i^+$, we use augmentation methods~\citep{chen2020simple} rather than maintaining multiple bins of each class~\citep{sohn2016improved}. 
Our framework can be effectively applied to class-imbalanced environment without any class information, due to the inductive bias introduced by augmentation.

Algorithm \ref{alg-DUEL-framework} summarizes the training procedure of DUEL framework. We utilize the InfoNCE loss $\gL_{\text{InfoNCE}}$ instead of $\gL_{\text{M-HML}}$ to train the feature extractor because $\gL_{\text{InfoNCE}}$ is equivalent to $\gL_{\text{M-HML}}$ under certain conditions: (1) $\qf{i}{j}=\exp((f(x_i)^\top f(x_j)-1)/\tau)$ and (2) $\lambda=1$. More details and mathematical support regarding the relationship between our DUEL framework and conventional SSL models are provided in the Appendix B.

After each training step of the feature extractor, the duplication elimination step begins. In the duplication elimination step, selected data according to the policy $\pi_{\text{DUEL}}$ is replaced with the current data. These two steps repeat iteratively until the termination condition is satisfied. We also provide the general form of memory-integrated Hebbian Metric Learning algorithm in Appendix C.

\begin{table*}[t]
    \renewcommand*{\arraystretch}{1.2}
    \begin{center}
        \begin{adjustbox}{width=\textwidth}
            \begin{tabular}{c|cccc|cccc}
            \Xhline{3\arrayrulewidth}
            \multirow{3}{*}{Method} &\multicolumn{4}{c|}{\bf STL-10}&\multicolumn{4}{c}{\bf CIFAR-10} \\
            &\multicolumn{4}{c|}{Class Probability $\rho_{\max}{\scriptstyle (\rho_{\min})}$}&\multicolumn{4}{c}{Class Probability $\rho_{\max}{\scriptstyle (\rho_{\min})}$} \\
            &0.1 {\small(0.1)}&0.25 {\small(0.083)}&0.5 {\small(0.056)}&0.75 {\small(0.028)}&0.1 {\small(0.1)}&0.25 {\small(0.083)}&0.5 {\small(0.056)}&0.75 {\small(0.028)}\\
            \hline
            MoCoV2&${\bf79.59}{\scriptstyle\pm6.57}$&${\bf79.32}{\scriptstyle\pm1.02}$&$77.28{\scriptstyle\pm2.68}$&$75.34{\scriptstyle\pm1.11}$&$80.99{\scriptstyle\pm1.85}$&${\bf82.50}{\scriptstyle\pm1.12}$&$76.54{\scriptstyle\pm1.97}$&$70.12{\scriptstyle\pm1.40}$\\
            SimCLR&$70.80{\scriptstyle\pm1.21}$&$69.59{\scriptstyle\pm0.50}$&$67.20{\scriptstyle\pm1.10}$&$68.75{\scriptstyle\pm2.08}$&$82.28{\scriptstyle\pm1.26}$&$78.90{\scriptstyle\pm1.60}$&$76.67{\scriptstyle\pm1.96}$&$72.81{\scriptstyle\pm1.39}$\\
            Barlow Twins&$77.48{\scriptstyle\pm3.10}$&$78.21{\scriptstyle\pm2.82}$&$72.47{\scriptstyle\pm0.56}$&$71.98{\scriptstyle\pm2.12}$&$51.72{\scriptstyle\pm0.30}$&$51.59{\scriptstyle\pm1.44}$&$54.89{\scriptstyle\pm0.93}$&$53.54{\scriptstyle\pm0.99}$\\
            BYOL&$68.42{\scriptstyle\pm0.80}$&$64.04{\scriptstyle\pm1.41}$&$58.78{\scriptstyle\pm3.49}$&$61.82{\scriptstyle\pm2.77}$&$67.87{\scriptstyle\pm3.26}$&$69.32{\scriptstyle\pm3.52}$&$67.50{\scriptstyle\pm1.63}$&$60.40{\scriptstyle\pm2.40}$\\
            \hline
            D-MoCo (ours)&${\bf79.20}{\scriptstyle\pm3.56}$&$76.03{\scriptstyle\pm2.06}$&${\bf78.65}{\scriptstyle\pm2.23}$&${\bf77.23}{\scriptstyle\pm0.48}$&$82.58{\scriptstyle\pm3.48}$&$81.40{\scriptstyle\pm2.24}$&${\bf80.53}{\scriptstyle\pm3.04}$&${\bf77.99}{\scriptstyle\pm1.71}$\\
            D-SimCLR (ours)&$75.89{\scriptstyle\pm2.90}$&$72.68{\scriptstyle\pm3.53}$&$78.23{\scriptstyle\pm4.38}$&$74.13{\scriptstyle\pm1.04}$&${\bf82.82}{\scriptstyle\pm1.10}$&${\bf82.37}{\scriptstyle\pm1.53}$&$79.56{\scriptstyle\pm2.61}$&$75.17{\scriptstyle\pm3.49}$\\
            \Xhline{3\arrayrulewidth}
            \end{tabular}
        \end{adjustbox}
        \caption{Linear probing accuracies with various settings. (3 times, \%)}
        \label{table1}
    \end{center}
\end{table*}

\section{Experiments}

The goal of our framework is to learn a robust representation given unrefined and instantaneous data sampled from a class-imbalanced distribution. Thus, we validate our framework in class-imbalanced environments.

\paragraph{Experiment setting}{In our experiments, we use the ResNet-50~\citep{he2016deep} as a backbone of the feature extractor. We choose MoCoV2~\citep{chen2020improved}, SimCLR~\citep{chen2020simple}, and Barlow Twins~\citep{zbontar2021barlow} as baselines. We implement our DUEL frameworks based on MoCoV2 and SimCLR by adding the DUEL process, denoted as D-MoCo and D-SimCLR respectively. Hyperparameters for all models are unified for fair comparison.
After the training, we evaluate each model with downstream tasks such as linear probing with class-balanced datasets to prove each model can extract generalized representations. More details for hyperparameters and the experiments are provided in the Appendix E.
}

\paragraph{Class-imbalanced environment}{We design a two-step data generator with predefined datasets to describe a class-imbalanced environment.
A dataset $D$ is partitioned into $D_c$ with each class $c\in\gC$. In every experiment, we assume that one class, denoted as $c_{\max}$, occurs much more frequently than others. The occurrence probability of the most frequent class is denoted as $\rho_{\max}$. The probabilities of the remaining classes are identically set as $\rho_{\min}=\frac{1}{|C|-1}(1-\rho_{\max})$.
Then the data is sampled from the environment in two steps:
\begin{enumerate}
\item \textit{Class Sampling}: $c\sim\rho,\: \rho_c=\begin{cases}\rho_{\max}\quad\quad c=c_{\max}\\\rho_{\min}\quad\quad\, c\neq c_{\max}\end{cases}$
\item \textit{Data Sampling}: $x\sim \text{Uniform}(D_{c})$.
\end{enumerate}
We utilize CIFAR-10~\citep{krizhevsky2009learning} and STL-10~\citep{coates2011analysis} for experiments. We also use ImageNet-LT~\citep{liu2019large}, which has a long-tailed class distribution, to validate our framework in a more realistic environment. See Appendix E for more details.}

\begin{figure}[tb]
     \centering
     \begin{subfigure}[b]{0.85\columnwidth}
         \centering
         \includegraphics[width=\columnwidth]{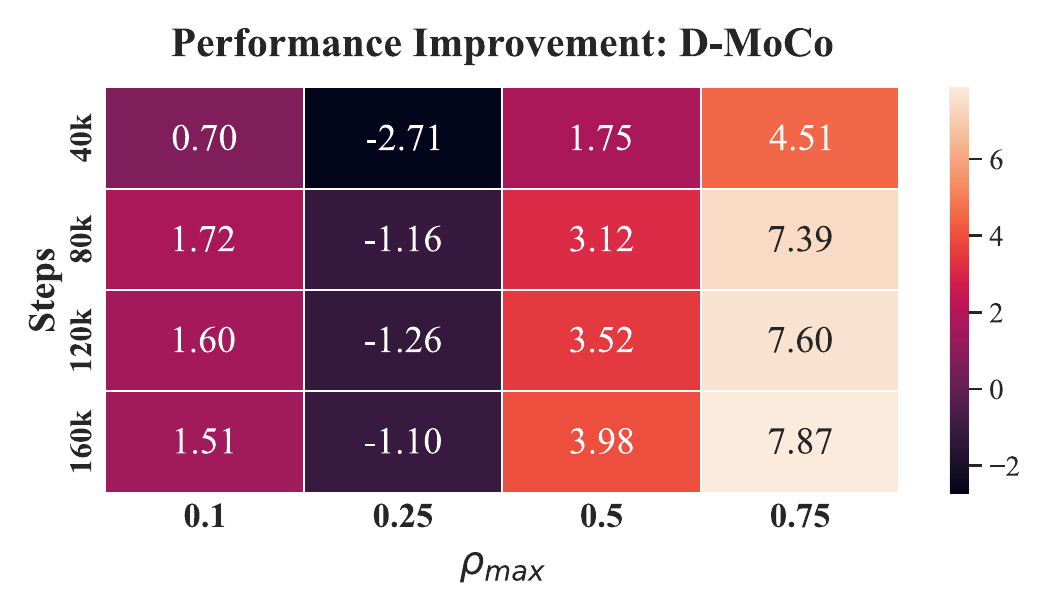}
         \caption{D-MoCo (Compared to MoCoV2)}
         \label{perform_moco}
     \end{subfigure}
     \hfill
     \begin{subfigure}[b]{0.85\columnwidth}
         \centering
         \includegraphics[width=\columnwidth]{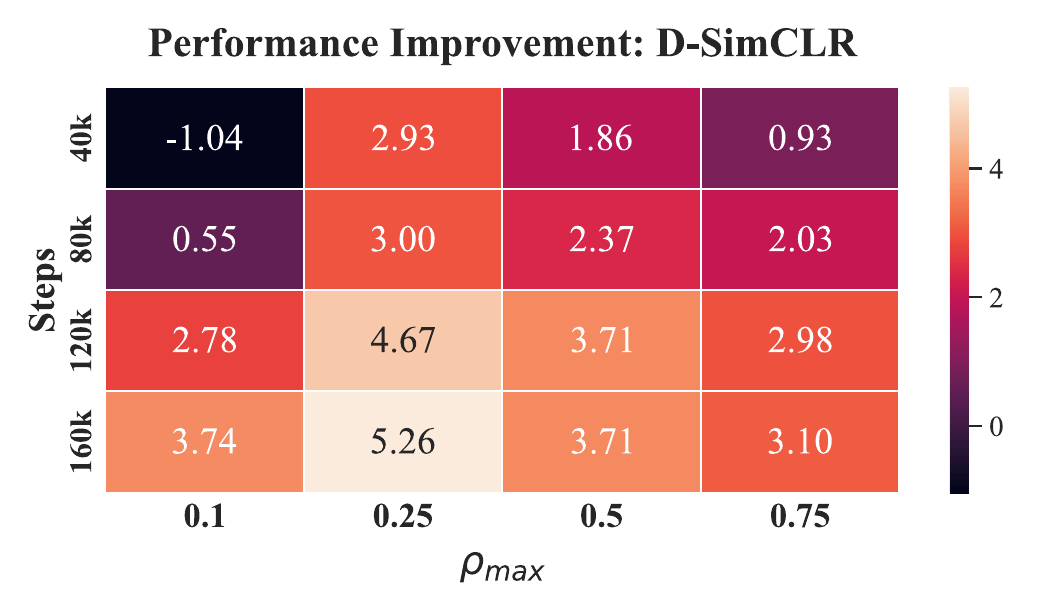}
         \caption{D-SimCLR (Compared to SimCLR)}
         \label{perform_simclr}
     \end{subfigure}
        \caption{Visualization of the performance enhancement in the linear probing task. In both D-MoCo and D-SimCLR, accuracies are gradually improved during the training steps. Especially in D-MoCo, the DUEL process can prevent the dramatical performance degradation with high $\rho_{\max}$.}
        \label{perform_improvement}
\end{figure}

\paragraph{Class-imbalanced learning with SSL frameworks}{To validate our approaches, we first conduct experiments with conventional SSL models in class-imbalanced environments. Table \ref{table1} shows that SSL models suffer from the performance degradation, especially when the class distribution is highly imbalanced. However, our frameworks can prevent the performance loss compared to their origin models. The visualization of the performance improvement during the training process in Figure \ref{perform_improvement} implies that the DUEL process gradually improves the robustness of representations. More experiments and results are described in Appendix F.
}

\begin{table}[t]
    \renewcommand*{\arraystretch}{1.2}
    \centering
    \begin{adjustbox}{width=1.0\columnwidth}
    \begin{tabular}{c|cccc}
    \Xhline{3\arrayrulewidth}
    \multirow{2}{*}{Metric}& \multirow{2}{*}{Method}& \multicolumn{3}{c}{Class Probability $\rho_{\max}{\scriptstyle (\rho_{\min})}$} \\
    &&0.1 {\small(0.1)}&0.5 {\small(0.056)}&0.75 {\small(0.028)} \\ 
    \hline
    \multirow{2}{*}{\begin{tabular}[c]{@{}c@{}}Class\\ Entropy ($\uparrow$)\end{tabular}} 
    & MoCo & $2.2991$ & $1.8394$ & $1.0523$ \\
    &\cellcolor{TableGray}  D-MoCo &\cellcolor{TableGray}  $2.2988$ &\cellcolor{TableGray}  ${\bf2.1654}$ &\cellcolor{TableGray}  ${\bf1.8306}$ \\ 
    \hline
    \multirow{2}{*}{\begin{tabular}[c]{@{}c@{}}Intra-class\\ Variance ($\downarrow$)\end{tabular}} 
    & MoCo   & $0.7879$ & $0.7853$ & $0.7689$ \\
    &\cellcolor{TableGray} D-MoCo &\cellcolor{TableGray}  $0.7750$ &\cellcolor{TableGray}  $0.7633$ &\cellcolor{TableGray}  $0.7699$ \\ 
    \hline
    \multirow{2}{*}{\begin{tabular}[c]{@{}c@{}}Inter-class\\ Similarity ($\downarrow$)\end{tabular}} 
    & MoCo   & $-0.1005$ & $-0.0534$ & $0.0723$ \\
    &\cellcolor{TableGray}  D-MoCo &\cellcolor{TableGray}  ${\bf-0.1033}$ &\cellcolor{TableGray}  ${\bf-0.0846}$ &\cellcolor{TableGray}  ${\bf-0.0513}$ \\
    \Xhline{3\arrayrulewidth}
    \end{tabular}
    \end{adjustbox}
    \caption{Quantitative analysis of the behavior of MoCo and D-MoCo with various metrics. (CIFAR-10)}
    \label{table2}
\end{table}

\begin{figure*}[!bt]
    \begin{center}
    \begin{subfigure}[b]{0.27\textwidth}
    \includegraphics[width=\columnwidth]{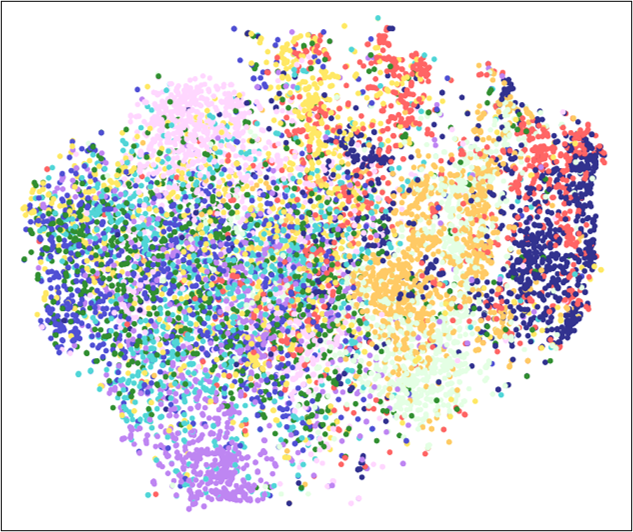}
    \caption{$t$-SNE (D-MoCo).}
    \label{tsne_all}
    \end{subfigure}%
    \hspace{0.7em}
    \begin{subfigure}[b]{0.27\textwidth}
    \includegraphics[width=\columnwidth]{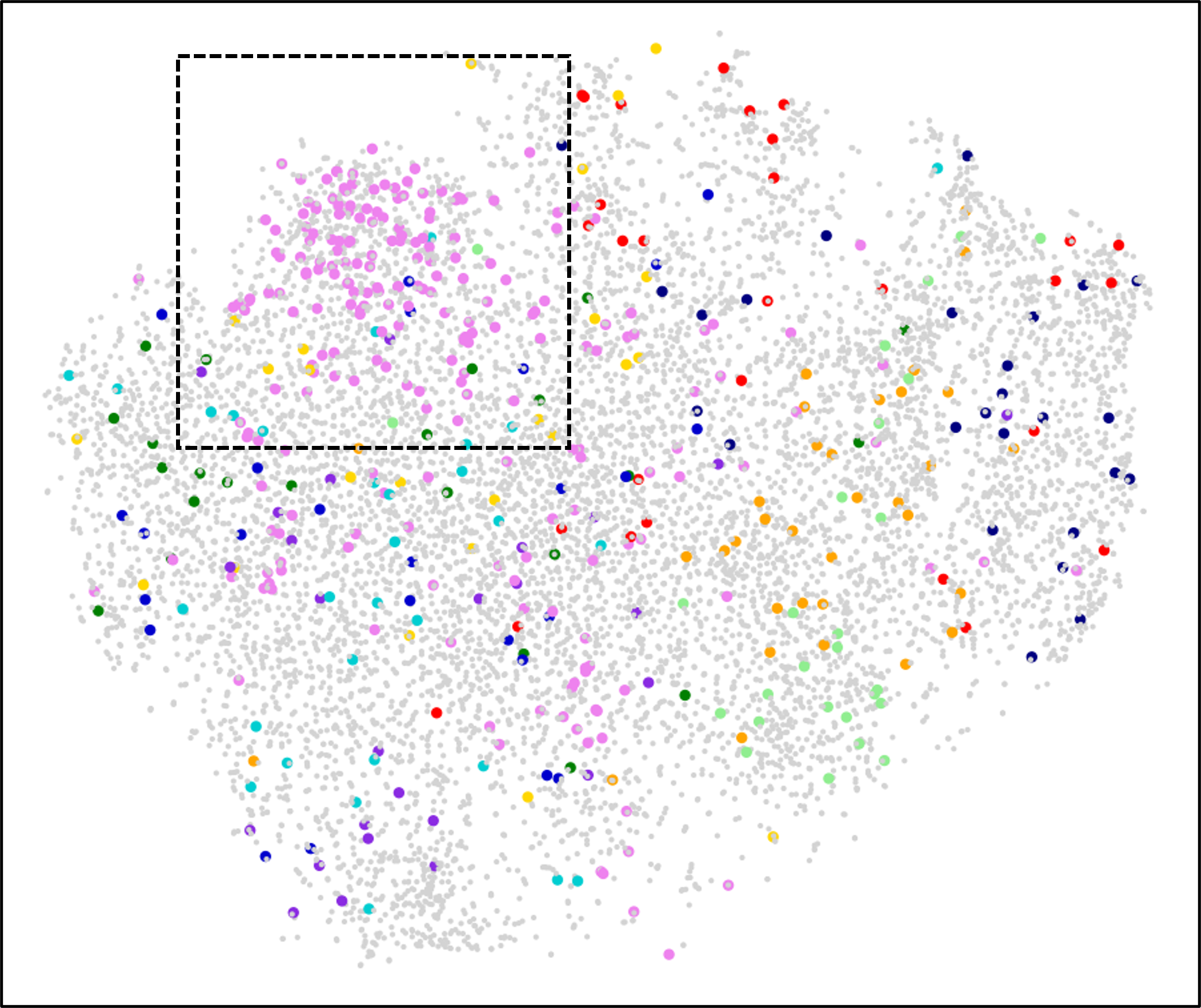}
    \caption{Batch-wise observation with $\gD'$.}
    \label{tsne_batch}
    \end{subfigure}%
    \hspace{0.7em}
    \begin{subfigure}[b]{0.27\textwidth}
    \includegraphics[width=\columnwidth]{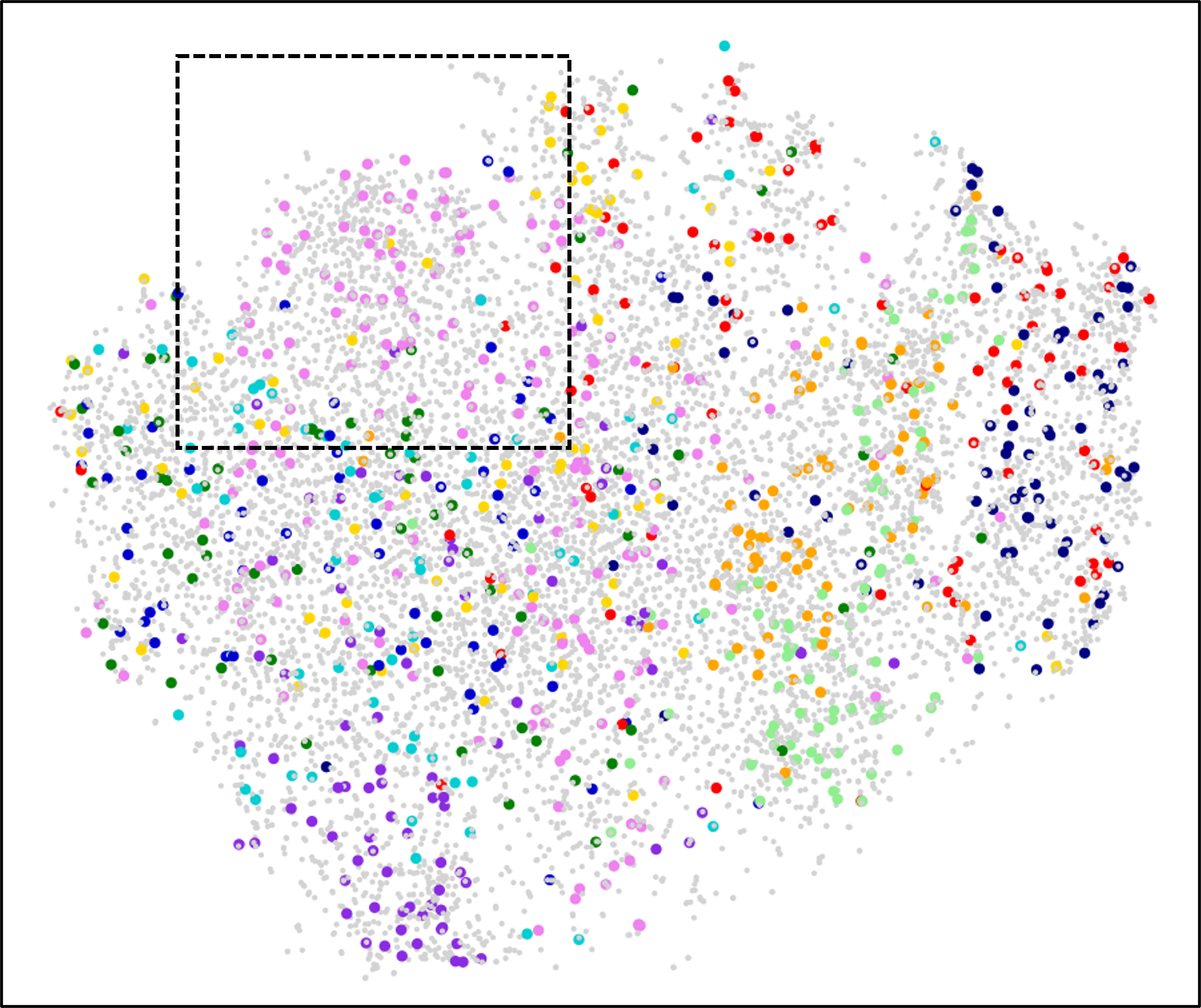}
    \caption{Updated memory with $\pi_{\text{DUEL}}$.}
    \label{tsne_memory}
    \end{subfigure}
    \end{center}
    \caption{t-SNE visualization of the active data filtering process with DUEL policy. (a) The representations extracted by the trained model along with their corresponding class. (b) The agent faces a dominant class (pink) that occurs more frequently than others. (c) The DUEL policy $\pi_{\text{DUEL}}$ replaces duplicated data with newer data and maximizes the distinctiveness information.}
    \label{tsne}
\end{figure*}

\paragraph{Analysis of the robustness of representation}{
Additionally, we measure and compare how well the representations extracted from the DUEL framework and the baseline's feature extractor cluster are formed. We use intra-class variance and inter-class similarity as measurements for this purpose. The intra-class variance and inter-class similarity are described in Equation \ref{v-intra} and \ref{d-inter}, respectively.
}
\begin{align}
\label{v-intra}
\bar{\text{v}}_{\text{intra}}:=&\frac{1}{|\gC|}\sum_{c\in\gC} \mathbb{E}_{x_c\sim \gD_c} \left[ (\bar{r}_{c}^\top f(x_c)-1)^2 \right]\\ \bar{\text{s}}_{\text{inter}}:=&\frac{1}{|\gC|\cdot(|\gC|-1)}\sum_{c\in\gC}\sum_{c'\neq c}(\bar{r}_{c}^\top \bar{r}_{c'})
\label{d-inter}
\end{align}
$\bar{r}_{c}$ is a centroid of each class on a hypersphere: $\bar{r}_{c}=\mathbb{E}_{x_c\sim \gD_c} \left[f(x_c)\right] / ||\mathbb{E}_{x_c\sim \gD_c} \left[f(x_c)\right]||_2$. Intra-class variance indicates how densely the representations of the same class are gathered, while inter-class similarity indicates how far apart the centroids of each class are. From the perspective of a classification task, both low intra-class variance and low inter-class similarity signify the robustness of representations. Table \ref{table2} presents the quantitative results for MoCo and D-MoCo. In both cases, the intra-class variance is preserved in every environment. However, the inter-class similarity of MoCo dramatically increases in extreme situation with $\rho_{\max}=0.75$. It implies that our framework extracts more distinguishable representation than MoCo when the data is class-imbalanced.

\paragraph{The role of the DUEL policy}{We analyze the properties of the data stored in the memory to show that the DUEL policy can effectively mitigate the class imbalances. In Table \ref{table2}, we observe that the entropy of the class distribution within the memory of D-MoCo is consistently higher than that of MoCo. This indicates that the DUEL process can maintain the diversity of the class information even in extremely class-imbalanced environments. We also visualize the policy of the DUEL framework with t-SNE~\citep{van2008visualizing} in Figure \ref{tsne}. Even in the presence of the frequent class (pink) (Figure \ref{tsne_batch}), the proposed framework filters out the duplicates and stores diverse data in the memory (Figure \ref{tsne_memory}).}

\section{Related Work}
\label{relatedwork}
\paragraph{Class-imbalanced learning}{
Class-imbalanced learning is a methodology for effective learning when class information is unevenly distributed in the data. To address this challenge, various techniques such as data resampling to smooth out class distributions~\citep{buda2018systematic,pouyanfar2018dynamic} and specialized loss functions~\citep{cao2019learning,cui2019class,tan2020equalization} have been employed. However, these approaches have limitations, as they require class information for each data point and may struggle to perform stably in extremely class-imbalanced environments. Recent research in class-imbalanced learning~\citep{yang2020rethinking,liu2021self} has shown that the self-supervised pretraining technique is more robust in class-imbalanced environments, even without explicit class information. To improve adaptation to extreme class-imbalanced environments, we have proposed an SSL framework with an additional active memory.
}

\paragraph{Self-supervised learning}{
Self-supervised learning has been proposed in different paradigms depending on the loss function and model architecture. For example, InfoNCE-based SSL models~\citep{chen2020simple,chen2020improved,oord2018representation}, for instance, can be considered as an extension of traditional metric learning that does not use class information. In the case of BYOL~\citep{grill2020bootstrap}, training process is based on knowledge distillation on the student model with a teacher model that is updated with momentum. Recently, several methods have been introduced, including Barlow Twins~\citep{zbontar2021barlow}, which perform metric learning by matching distributions on the latent space~\citep{bardes2021vicreg,liu2022self,chen2021exploring}. To validate the DUEL framework, we compared representative models from each paradigm, primarily based on InfoNCE.
}

\paragraph{Dealing with negative samples}{
In contrastive SSL, numerous studies have highlighted the significant impact of properly configuring negative samples on the model performance. Since self-supervised learning fundamentally extracts negative samples in an i.i.d. manner, the influence of the number of negative samples on training has been investigated~\citep{arora2019theoretical,ash2021investigating,awasthi2022more}. Subsequently, techniques such as generating virtual data using interpolation between samples~\citep{kalantidis2020hard} and applying penalties to elements within negative samples that share the same class information for debiasing have been employed~\citep{chuang2020debiased}. In addition, methods using mutual dependencies among elements within the same batch to adjust the degree of learning for each triplet have also been proposed~\citep{tian2022understanding}. Our filtering algorithm has improved performance by encouraging the maximization of distinctiveness information among negative samples, especially for data distributions containing class imbalances.
}

\section{Conclusion}

With respect to self-supervised class-imbalanced learning, we mainly claim that an active memory is essential to robustly generalize to instantaneous and class-imbalanced data without class information. We first introduce the Hebbian Metric Learning which optimizes both distinctiveness and Hebbian information. As an implementation of memory-integrated HML, we propose the Duplicate Elimination framework inspired by the working memory. We validate the DUEL framework with class-imbalanced environments and analyze the behavior of the framework. Our novel framework gradually maximizes the distinctiveness information in the memory, which leads to the preservation of the robustness despite dramatic class imbalance.

\paragraph{Limitations}{As we discuss, finding the optimal memory management policy $\pi^*$ is difficult to achieve in practice. Although the DUEL policy provides sufficient robustness, one can argue that our policy does not perform \textit{optimally} in some situations. We claim that further investigations on HML and distinctiveness information will be pivotal in comprehending the behavior of SSL and determining the best policy.
}

\section{Acknowledgments}
This work was partly supported by the IITP (2021-0-02068-AIHub/15\%, 2021-0-01343-GSAI/20\%, 2022-0-00951-LBA/25\%, 2022-0-00953-PICA/25\%) and NRF (RS-2023-00274280/15\%) grant funded by the Korean government.

\bibliography{aaai24}

\clearpage
\onecolumn
\appendix
\begin{center}
\Large \bf
Supplementary Materials of \\
DUEL: Duplicate Elimination on Active Memory for Self-supervised Class-imbalanced Learning
\end{center}

\section{A. Mathematical Support}
\setcounter{thm}{0}
\setcounter{lemma}{0}
\setcounter{prop}{0}

\begin{defn}[Message passing] We can generalize the class information of unknown samples with known samples and feature extractor $f$. By using mutual duplication probability and message passing with them.
\label{message_passing}
\end{defn}
\begin{equation*}
q(c_i=c|x_i;f):=\frac{1}{Z} \mathbb{E}_{x_j\sim \gD} \left[ \qf{i}{j}p(c_{j}=c|x_{j}) \right]
\end{equation*}
\begin{equation*}
Z=\sum_{c\in \gC} \mathbb{E}_{x_k\sim \gD} \left[ \qf{i}{k}p(c_{k}=c|x_{k}) \right]
\end{equation*}
$Z$ is a normalization factor over class information $c\in\gC$.

\begin{lemma}[Joint distribution with density function]
Joint distribution $q(x_i,c_i=c;f)$ is derived with $p(c_j=c|x_j)$ and duplication density function $Q(x_i,x_j;f)$.
\begin{equation*}
q(x_i,c_i=c;f)=\int_{\gX} p(c_{j}=c|x_{j})Q(x_i,x_j;f)dx_j
\end{equation*}
\label{lemma1}
\end{lemma}
\begin{proof}
\begin{equation*}
q(x_i,c_i=c;f)=q(c_i=c|x_i;f)p(x_{i})
\end{equation*}
\begin{equation*}
=\frac{1}{Z} \mathbb{E}_{x_j\sim \gD} \left[ \qf{i}{j} p(c_{j}=c|x_{j}) \right]p(x_{i}) \qquad\small{\leftarrow\:\text{Definition \ref{message_passing}}}
\end{equation*}
\begin{equation*}
Z=\mathbb{E}_{x_k\sim \gD} \left[ \qf{i}{k}\sum_{c\in \gC} p(c_{k}=c|x_{k}) \right]=\mathbb{E}_{x_k\sim \gD} \left[ \qf{i}{k}\right]
\end{equation*}
\begin{equation*}
\therefore q(x_i,c_i=c;f)=\mathbb{E}_{x_j\sim \gD} \left[ \frac{\qf{i}{j}}{\mathbb{E}_{x_k\sim \gD} \left[ \qf{i}{k}\right]}p(c_{j}=c|x_{j}) \right]p(x_i)
\end{equation*}
\begin{equation*}
=\int_{\gX} p(c_{j}=c|x_{j})Q(x_i,x_j;f)dx_j
\end{equation*}
Similary, we can compute $p(c_i=c,x_i)$.
\begin{equation*}
p(x_i,c_i=c)=\int_{\gX} p(c_{j}=c|x_{j})P(x_i,x_j)dx_j
\end{equation*}
\end{proof}

\begin{prop}[Hebbian Metric Learning] For every feature extractor $f$, minimizing $D_{\text{KL}}(p(x,c)||q(x,c;f))$ is equivalent to minimizing $\gL_{\text{HML}}(f;\gD)$, which can be derived as:
\begin{align*}
\small \argmin{f}D_{\text{KL}}(p||q) &=\argmin{f}D_{\text{KL}}(P||Q) \\
&=\argmin{f}\mathbb{E}_{x_i\sim \gD} \left[\mathbb{E}_{x_j\sim \gD^{+}_{i}} \left[-\log\frac{\qf{i}{j}}{\mathbb{E}_{x_k\sim \gD} \left[ \qf{i}{k}\right]}\right]\right] 
\\
&=\argmin{f}\underbrace{\left(\gI_h(f;\gD)-\gI_d(f;\gD)\right)}_{\gL_{\text{HML}}(f;\gD)}
\end{align*}
\end{prop}
\begin{proof}
By Lemma \ref{lemma1}, the joint distribution $p$ and $q$ are the integrated form of $P$ and $Q$. Thus we can optimize on $q$ by optimizing $Q$ on $P$.
\begin{equation*}
\therefore \argmin{f}D_{\text{KL}}(p||q) =\argmin{f}D_{\text{KL}}(P||Q)
\end{equation*}
\begin{align*}
&\small=\argmin{f}\int_{\gX}\int_{\gX} - \frac{\pf{i}{j}p(x_i)p(x_j)}{\mathbb{E}_{x_k\sim \gD} \left[ \pf{i}{k}\right]}\log\frac{\qf{i}{j}p(x_i)p(x_j)}{\mathbb{E}_{x_k\sim \gD} \left[ \qf{i}{k}\right]}dx_i dx_j\\
&\small=\argmin{f}\int_{\gX}\left(\int_{\gX} - \frac{\pf{i}{j}p(x_j)}{\mathbb{E}_{x_k\sim \gD} \left[ \pf{i}{k}\right]}\log\frac{\qf{i}{j}}{\mathbb{E}_{x_k\sim \gD} \left[ \qf{i}{k}\right]} dx_j\right) p(x_i)dx_i
\end{align*}
Note that the coefficient of the log term can act as a probability density function.
\begin{equation}
\because \int_{\gX}\frac{\pf{i}{j}p(x_j)}{\mathbb{E}_{x_k\sim \gD} \left[ \pf{i}{k}\right]}dx_j=\frac{\mathbb{E}_{x_j\sim \gD} \left[ \pf{i}{j}\right]}{\mathbb{E}_{x_k\sim \gD} \left[ \pf{i}{k}\right]}=1
\label{prob_d_plus}
\end{equation}
Let the probability in Equation \ref{prob_d_plus} be $\gD_i^+$. Then the RHS can be rewritten with $\gD_i^+$.
\begin{align*}
\text{RHS}&=\argmin{f}\int_{\gX}\mathbb{E}_{x_j\sim \gD_i^+} \left[ -\log\frac{\qf{i}{j}}{\mathbb{E}_{x_k\sim \gD} \left[ \qf{i}{k}\right]} \right] p(x_i)dx_i\\
&=\argmin{f}\mathbb{E}_{x_i\sim \gD} \left[\mathbb{E}_{x_j\sim \gD_i^+} \left[ -\log\frac{\qf{i}{j}}{\mathbb{E}_{x_k\sim \gD} \left[ \qf{i}{k}\right]} \right] \right]\\
&=\argmin{f}\mathbb{E}_{x_i\sim \gD} \left[\mathbb{E}_{x_j\sim \gD_i^+} \left[ -\log \qf{i}{j} \right] + -\log\mathbb{E}_{x_k\sim \gD} \left[ \qf{i}{k}\right] \right]\\
&=\argmin{f}\mathbb{E}_{x_i\sim \gD} \left[I_h(x_i;f,\gD)-\gI_d(x_i;f,\gD)\right]\\
&=\argmin{f}\left(\gI_h(f;\gD)-\gI_d(f;\gD)\right).
\end{align*}
\end{proof}

\begin{coro}[Optimality of HML] Let a class distribution $p(c)$ of $\gD$ is evenly distributed:$\forall c,p(c)=1/|\gC|$. The optimal feature extractor $f^*$, which minimizes $\left(\gI_h(f;\gD)-\gI_d(f;\gD)\right)$ is derived as:
\begin{equation}
\small \qfo{i}{j}=\begin{cases}1\qquad c_i=c_j\\0\qquad c_i\neq c_j\end{cases}
\end{equation}
where the value of the objective function $\gL_{\text{HML}}(f^*;\gD)$ is $-\log |\gC|$. In other words, the lower bound of $\gL_{\text{HML}}(f;\gD)$ is formulated as:
\begin{align}
\small \gL_{\text{HML}}(f;\gD)\ge -\log |\gC|.
\end{align}
\label{cor1}
\end{coro}
\begin{proof}
We first find $f$ that minimizes $\gI_h(f;\gD)$. We define $\gD_c$ as the data distribution when the class is given by $c$. We define $\gI_{h}(f;\gD_c)$ with $\gD_c$.
\begin{equation*}
\gI_{h}(f;\gD_c):=\mathbb{E}_{x_i\sim \gD_c} \left[\gI_{h}(x_i;f,\gD_c)\right]=\mathbb{E}_{x_i\sim \gD_c} \left[\gI_{h}(x_i;f,\gD)\right]
\end{equation*}
\begin{align*}
\gI_{h}(f;\gD)&=\mathbb{E}_{x_i\sim \gD} \left[\gI_{h}(x_i;f,\gD)\right] \\
&=\mathbb{E}_{c\sim U_\gC} \left[\mathbb{E}_{x_i\sim \gD_c} \left[\gI_{h}(x_i;f,\gD)\right]\right] \\
&=\mathbb{E}_{c\sim U_\gC} \left[ \gI_{h}(f;\gD_c) \right].
\end{align*}
Therefore, $\gI_{h}(f;\gD)$ becomes zero if $\forall c, \gI_{h}(f;\gD_c)=0$. This means that $\qf{i}{j}=1$ is satisfied when $c_i=c_j$. In this case, $q$ is expressed as follows:
\begin{equation*}
\small \qf{i}{j}=\begin{cases}q_{c_ic_i}\qquad c_i=c_j\\q_{c_ic_j}\qquad c_i\neq c_j\end{cases}
\end{equation*}
where $q_{c_ic_i}=1$. Lastly, we find $f$ that maximizes $\gI_d(f;\gD)$.
\begin{align*}
-\gI_d(f;\gD)&=-\mathbb{E}_{x_i\sim \gD} \left[-\log\mathbb{E}_{x_k\sim \gD} \left[ \qf{i}{k}\right] \right]\\
&=-\mathbb{E}_{c\sim U_\gC} \left[\mathbb{E}_{x_i\sim \gD_c} \left[-\log\left(\mathbb{E}_{c'\sim U_\gC} \left[\mathbb{E}_{x_k\sim \gD_{c'}} \left[ \qf{i}{k}\right] \right]\right)\right]\right]\\
&=-\mathbb{E}_{c\sim U_\gC} \left[-\log\left(1/|\gC|\cdot\Sigma_{c'\in\gC}q_{c'c}\right)\right]\\
&=\mathbb{E}_{c\sim U_\gC} \left[\log\left(1/|\gC|\cdot(1+\Sigma_{c'\neq c}q_{c'c})\right)\right]\\
&\ge\log(1/|\gC|)=-\log|\gC|. \qquad\qquad\small{\leftarrow\:\forall c'\neq c,\:q_{c'c}=0}
\end{align*}

Therefore, $\gL_{\text{HML}}(f;\gD)\ge -\log|\gC|$ is satisfied, and the equality is satisfied when $f=f^*$.

\end{proof}

\begin{prop}[HML Bound] Let $\gD$ and $\gD'$ be the oracle and the empirical data distribution, respectively. Then the upper bound of ideal HML loss is formulated as below:
\begin{align*}
\small \gL_{\text{HML}}(f;\gD) \le \underbrace{\lambda \cdot \gI_{h}(f;\mathcal{D'})-\gI_d(f;\gD',\gM)+|\gI_d(f;\gD',\gM)-\gI_d(f;\gD)|}_{\gL_{\text{M-HML}}(f,\gM;\gD')}
\end{align*}
where $\gI_d(f;\gD',\gM)=\mathbb{E}_{x_i\sim \gD'} \left[\gI_d(x_i;f,\gM)\right]$ denotes the empirical distinctiveness information in the memory $\gM$ and $\lambda=1/(|\gC|\cdot\rho_{\min})$.
\end{prop}
\begin{proof}
We first prove that $\gI_h(f;\gD)\le \lambda \cdot \gI_{h}(f;\gD')$. Suppose there exists $\rho$, which is a class distribution of the class-imbalanced data distribution $\gD'$. The class distribution of $\gD$ is a uniform distribution with probability $1/|\gC|$. Suppose $\gD$ and $\gD'$ have the same $\gD_c$ for all $c$. This means that the only difference between $\gD$ and $\gD'$ is the class distribution $\rho$.
\begin{align*}
\gI_h(f;\gD)=\sum_{c\in\gC}\left(\frac{1}{|\gC|}\cdot \gI_{h}(f;\gD_c)\right)
\end{align*}
Similiary, we compute $\gI_{h}(f;\gD')$.
\begin{equation*}
\gI_{h}(\gD';f)=\sum_{c\in\gC}\left(\rho_c\cdot \gI_{h}(f;\gD_c)\right)
\end{equation*}
Let the minimum probability of $\rho$ be $\rho_{\min}$. This means that $\rho_c/\rho_{\min} \ge 1$ for all $c\in\gC$.
\begin{equation*}
\therefore \gI_h(f;\gD)\le \frac{1}{|\gC|\cdot\rho_{\min}}\cdot \gI_{h}(f;\gD')=\lambda\cdot \gI_{h}(f;\gD')
\end{equation*}
\begin{align*}
\gI_h(f;\gD)-\gI_d(f;\gD) &\le \lambda \cdot \gI_{h}(f;\mathcal{D'})-\gI_d(f;\gD) \\
&=\lambda \cdot \gI_{h}(f;\mathcal{D'})-\gI_d(f;\gD',\gM)+\gI_d(f;\gD',\gM)-\gI_d(f;\gD) \\
&\le \lambda \cdot \gI_{h}(f;\mathcal{D'})-\gI_d(f;\gD',\gM)+|\gI_d(f;\gD',\gM)-\gI_d(f;\gD)|.
\end{align*}
\end{proof}

\begin{thm}[Optimality of M-HML]
Assume that an optimal memory $\gM^*\simeq \gD$ exists. With the memory $\gM^*$, the ideal loss  $\gL_{\text{HML}}$ and empirical loss $\gL_{\text{M-HML}}$ shares the optimal feature extractor $f^*$:
\begin{equation*}
\gL_{\text{M-HML}}(f^*,\gM^*;\gD')=\gL_{\text{HML}}(f^*;\gD)=-\log |\gC|
\end{equation*}
where mutual duplication probability with $f^*$ satisfies the following property.
\begin{equation*}
\small \qfo{i}{j}=\begin{cases}1\qquad c_i=c_j\\0\qquad c_i\neq c_j\end{cases}
\end{equation*}
\end{thm}
\begin{proof}
To show that two losses have the same optimal feature extractor, we first derive the difference of two losses.
\begin{align*}
|\gL_{\text{M-HML}}(f,\gM;\gD')-\gL_{\text{HML}}(f;\gD)&|=|\lambda \cdot \gI_{h}(f;\mathcal{D'})-\gI_d(f;\gD',\gM)\\
&+|\gI_d(f;\gD',\gM)-\gI_d(f;\gD)|-\gI_h(f;\gD)+\gI_d(f;\gD)|
\end{align*}
\begin{equation*}
\le|\lambda \cdot \gI_{h}(f;\mathcal{D'})-\gI_h(f;\gD)|+2\cdot|\gI_d(f;\gD',\gM)-\gI_d(f;\gD)|.
\end{equation*}

Following the same proof technique of Corollary \ref{cor1}, both $\gI_{h}(f;\mathcal{D'})$ and $\gI_h(f;\gD)$ become zero when $\qf{i}{j}=1$ with $c_i=c_j$. We then show that $|\gI_d(f;\gD',\gM)-\gI_d(f;\gD)|$ becomes zero with the feature extractor $f^*$ and memory $\gM^*\simeq \gD$.

\begin{align*}
|\gI_d(f^*,\gM^*;\gD')-\gI_d(f^*;\gD)|&=\mathbb{E}_{c\sim \rho} \left[-\log\left(\Sigma_{c'\in\gC}\rho_{c'}^{(\gM^*)}q_{c'c}\right)\right]
-\mathbb{E}_{c\sim U_\gC} \left[-\log\left(1/|\gC|\cdot\Sigma_{c'\in\gC}q_{c'c}\right)\right] \\
&=\mathbb{E}_{c\sim \rho} \left[-\log\left(\rho_{c}^{(\gM^*)}\right)\right]-\mathbb{E}_{c\sim U_\gC} \left[-\log\left(1/|\gC|\right)\right] \qquad\small{\leftarrow\:\forall c'\neq c, q_{c'c}=0}\\
&=\mathbb{E}_{c\sim \rho} \left[-\log\left(1/|\gC|\right)\right]-\mathbb{E}_{c\sim U_\gC} \left[-\log\left(1/|\gC|\right)\right]\qquad\small{\leftarrow\:\gM^*\simeq \gD}\\
&=0
\end{align*}
\begin{equation*}
\therefore \gL_{\text{M-HML}}(f^*,\gM^*;\gD')=\gL_{\text{HML}}(f^*;\gD)=-\log |\gC|.
\end{equation*}

\end{proof}

\begin{defn}[\textit{Safeness} of data filtering] Let the memory $\gMp$ be updated as $\gMp'$ after the data filtering process with a policy $\pi$. A filtering process is \textit{safe} when the following inequality is satisfied when the optimal feature extractor $f^*$ is given.
\begin{equation*}
\gI_d(f^*,\gMp;\gD') \le \gI_d(f^*,\gMp';\gD')
\end{equation*}
\begin{equation*}
\qfo{i}{j}=\begin{cases}
1\qquad c_i=c_j\\
0\qquad c_i\neq c_j
\end{cases}
\end{equation*}
\label{safeness}
\end{defn}

\begin{prop}[\textit{Safeness} of DUEL policy] Let the memory $\gMp$ be updated as $\gMp'$ after the data filtering process with DUEL policy $\pi_{\text{DUEL}}$. This filtering process is \textit{safe} with $f^*$ which satisfies:
\begin{equation*}
\gI_d(f^*,\gMp;\gD') \le \gI_d(f^*,\gMp';\gD')
\end{equation*}
\end{prop}
\begin{proof}
First, we show which item is replaced with the policy $\pi_{\text{DUEL}}$. In this case, let the memory be initialized with $\gM\simeq\gD'$.
\begin{align*}
\argmin{j\in\{1..K\}}\gI_d(x_j;f^*,\gD')&=\argmin{j\in\{1..K\}}\left(-\log\mathbb{E}_{x_k\sim \gD'} \left[ \qf{j}{k}\right]\right)\\
&=\argmin{j\in\{1..K\}}\left(-\log\left(\rho_{c_j}\right)\right)\\
&=\argmin{j\in\{1..K\},c_j=c_{\max}}\left(-\log(\rho_{\max})\right).
\end{align*}
So we assume that the data to be replaced has the class $c_{\max}$. Since the representations with the same class are identical to $f^*$, the distinctiveness information is not changed if the new data $x_{\text{new}}$ has the class $c_{\max}$.

Let the class of the data $x_{\text{new}}$ be $c_{\text{new}}\neq c_{\max}$. During the replacement, the data with classes other than $c_{\text{new}}$ and $c_{\max}$ will remain in memory. Let $N_{c}$ be the number of elements in memory with class $c$. Then $\sum_{c\in\gC}N_{c}=K$ and $0<N_{c_{\max}}+N_{c_{\text{new}}}=K'\le K$ are satisfied.
\begin{equation*}
\gI_d(f^*,\gMp';\gD')-\gI_d(f^*,\gMp;\gD')=\rho_{\max}\left(-\log\frac{N_{c_{\max}}-1}{N_{c_{\max}}}\right)-\rho_{c_{\text{new}}}\left(-\log\frac{N_{c_{\text{new}}}+1}{N_{c_{\text{new}}}}\right)
\end{equation*}
\begin{align*}
&=\rho_{\max}\left(-\log\frac{N_{c_{\max}}-1}{N_{c_{\max}}}\right)-\rho_{c_{\text{new}}}\left(-\log\frac{K'-N_{c_{\max}}+1}{K'-N_{c_{\max}}}\right)\\
&\ge\rho_{c_{\text{new}}}\left(-\log\frac{N_{c_{\max}}-1}{N_{c_{\max}}}\right)-\rho_{c_{\text{new}}}\left(-\log\frac{K'-N_{c_{\max}}+1}{K'-N_{c_{\max}}}\right)\\
&=\rho_{c_{\text{new}}}\left(-\log\frac{(N_{c_{\max}}-1)(K'-N_{c_{\max}})}{N_{c_{\max}}(K'-N_{c_{\max}}+1)}\right)\\
&=\rho_{c_{\text{new}}}\left(-\log\frac{N_{c_{\max}}(K'-N_{c_{\max}}+1)-K'}{N_{c_{\max}}(K'-N_{c_{\max}}+1)}\right)>0
\end{align*}
\begin{equation*}
\therefore \gI_d(f^*,\gMp;\gD') \le \gI_d(f^*,\gMp';\gD')
\end{equation*}
\end{proof}

\section{B. Memory-integrated Hebbian Metric Learning and Conventional Metric Learning}
\label{apdx_remark_convention}

In this section, we provide the mathematical analyses to show the relationship between Hebbian Metric Learning and conventional metric learning. These analyses are based on the partial solution of Memory-integrated Hebbian Metric Learning with the condition $\gD'=\gD$, which is often assumed in the metric learning domain. The rewritten form of Equation \ref{emp_hml_ineq} with this condition is as follows.
\begin{align}
\argmin{f,\gM}(\gI_h(f;\gD)-\gI_d(f;\gD,\gM)+|\gI_d(f;\gD,\gM)-\gI_d(f;\gD)|)
\label{emp_hml_opt_same}
\end{align}
With $\qf{i}{j}=\exp{((f(x_i)^\top f(x_j)-1)/\tau)}$ in Equation \ref{emp_hml_opt_same}, InfoNCE loss function with $0\le\epsilon\le1$ can be derived from the Memory-integrated Hebbian Metric Learning.
\begin{align*}
\gI_h(f;\gD)&-\gI_d(f;\gD,\gM)\\
&=\mathbb{E}_{x_i\sim \gD} \left[\mathbb{E}_{x_j\sim \gD^{+}_{i}} \left[-\log \frac{K\cdot\qf{i}{j}}{\sum_{k=1}^K\qfmm{i}{k}}\right]\right]\\
&=\mathbb{E}_{x_i\sim \gD} \left[\mathbb{E}_{x_j\sim \gD^{+}_{i}} \left[-\log \frac{\qf{i}{j}}{\sum_{k=1}^K\qfmm{i}{k}}\right]\right]-\log K\\
&=\mathbb{E}_{x_i\sim \gD} \left[\mathbb{E}_{x_j\sim \gD^{+}_{i}} \left[-\log \frac{\exp{((f(x_i)^\top f(x_j)-1)/\tau)}}{\sum_{k=1}^K\exp{((f(x_i)^\top f(x_k)-1)/\tau)}}\right]\right]-\log K\\
&=\mathbb{E}_{x_i\sim \gD} \left[\mathbb{E}_{x_j\sim \gD^{+}_{i}} \left[-\log \frac{\exp{(f(x_i)^\top f(x_j)/\tau)}}{\sum_{k=1}^K\exp{(f(x_i)^\top f(x_k)/\tau)}}\right]\right]-\log K 
\end{align*}
\begin{equation*}
\small \le \mathbb{E}_{x_i\sim \gD} \left[\mathbb{E}_{x_j\sim \gD^{+}_{i}} \left[-\log \frac{\exp{(f(x_i)^\top f(x_j)/\tau)}}{\epsilon\cdot\exp{(f(x_i)^\top f(x_j)/\tau)}+\sum_{k=1}^K\exp{(f(x_i)^\top f(x_k)/\tau)}}\right]\right]-\log K
\end{equation*}

Unfortunately, the term $|\gI_d(f;\gD,\gM)-\gI_d(f;\gD)|$ has not been sufficiently discussed in previous work, although the former term has been defined in various forms. In metric learning, the data in memory $\gM$ are used as \textit{negative} samples, which comes from contrastive learning. On the other hand, self-supervised learning methods often use negative samples, which are drawn i.i.d. from the data distribution. This means that SSL uses $\gI_d(f;\gD,\gM)$ as an approximation of $\gI_d(f;\gD)$. As the amount of data in memory increases, $\gI_d(f;\gD,\gM)$ becomes more accurate, and this term will almost certainly become zero. At the same time, performance in downstream tasks gradually increases due to the decrease between $\gI_d(f;\gD,\gM)$ and $\gI_d(f;\gD)$. This supports the claim that reducing this term is important to improve the accuracy of downstream tasks in self-supervised learning.

\section{C. General Algorithm for Memory-integrated Hebbian Metric Learning}
\label{apdx_empirical_hml}

In this section, we discuss a general algorithm for implementing Memory-integrated Hebbian Metric Learning. 
Assuming that the data distribution changes over time, we define the data distribution as a function $\gD_{t}$ representing the distribution at time $t$. Since there is no termination condition for the agent's learning, we can represent a single learning process as an algorithm.

Algorithm \ref{alg-emp-hml} describes the updating of the feature extractor and the memory for a single iteration. At each time step $t$, the agent perceives new observations. The acquired data is then used to update the model parameter $\theta$ using the Equation \ref{emp_hml}. Then, using the memory management policy $\pi$, the memory $\gMp'$ is obtained by updating with the new data. However, this update may not be \textit{safe}, as described in Definition \ref{safeness}. If the distinctiveness information does not increase with this process, the memory is not updated.

\begin{algorithm}[ht]
    \textbf{Model : } feature extractor $f_\theta$, memory $\gMp$, memory management policy $\pi$ \\
    \textbf{Input : } data distribution $\gD_t$ with time $t$, batch size $B$, memory size $K$, learning rate $\eta$ \\
    \textbf{Output :} trained feature extractor $f_{\theta^*}$, updated memory $\gMp'$ 
    \begin{algorithmic}[1]
        \STATE $\{(d_{b},d_{b}^+)\}_{b=1}^B \leftarrow \text{Observation}(\gD_t)$ \hfill $\triangleright$ Observation from the Environment
        \STATE Compute $\gL_{\text{M-HML}}(\{d_{b}\}_{b=1}^B,\{d_{b}^+\}_{b=1}^B,\gMp;f_\theta)$ \hfill $\triangleright$ Memory-integrated HML with Equation \ref{emp_hml}
        \STATE $\theta\leftarrow \theta - \eta\nabla_{\theta}\gL_{\text{M-HML}}$
        \STATE $\gMp' \leftarrow \text{filter}(\gMp,\{d_{b}\}_{b=1}^B;\pi,f_\theta)$ \hfill $\triangleright$ Active Data Filtering with Equation \ref{mem_optimize}
        \IF {$\gI_d(f_\theta,\gMp;\gD_t) \le \gI_d(f_\theta,\gMp';\gD_t)$}
            \STATE $\gMp \leftarrow \gMp'$
        \ENDIF
    \end{algorithmic}
    \caption{Memory-integrated Hebbian Metric Learning}
    \label{alg-emp-hml}
\end{algorithm}

Additionally, for high reproducibility, we aim to present a generalized code used to train the frameworks in a PyTorch style. In this case, the data distribution is fixed because we focus on the class-imbalance in this work. Algorithm \ref{alg-pytorch-train} provides the pseudo-code of the implemented version. A single training iteration can be divided into three small steps. When new data is acquired, the initial step involves resampling additional data from memory, followed by optimizing the loss function using these samples. Subsequently, the process of updating the model and memory is added based on the model's architecture.

\begin{algorithm}[ht]
    \small
    \begin{algorithmic}[1]
        \STATEPTCOMMENT{0}{USING\_MEMORY: MoCo, D-MoCo, D-SimCLR}
        \STATEPTCOMMENT{0}{MOMENTUM\_UPDATE: MoCo, D-MoCo, BYOL}
        \STATEPTCODE{0}{for target, positive in loader:}
        \STATEPTCOMMENT{1}{Resampling from the memory}
        \STATEPTCODE{1}{if model.\textit{USING\_MEMORY}:}
        \STATEPTCODE{2}{negative = model.sample\_from\_memory(neg\_size)}
        \INLINEPTCOMMENT{Additional negative samples}
        \STATEPTCODE{1}{else:}
        \STATEPTCODE{2}{negative = None}
        \STATEPTBREAK
        \STATEPTCOMMENT{1}{Projection and compute loss function}
        \STATEPTCODE{1}{loss = model.forward(target, positive, negative)}
        \STATEPTCODE{1}{optimizer.zero\_grad()}
        \STATEPTCODE{1}{loss.backward()}
        \STATEPTCODE{1}{optimizer.step()}
        \STATEPTBREAK
        \STATEPTCOMMENT{1}{Model and memory update}
        \STATEPTCODE{1}{if model.\textit{MOMENTUM\_UPDATE}:}
        \STATEPTCODE{2}{model.\_momentum\_update()}
        \STATEPTCODE{1}{if model.\textit{USING\_MEMORY}:}
        \STATEPTCODE{2}{model.\_memory\_update(target)}
        \STATEPTBREAK
        \STATEPTCOMMENT{1}{Scheduler update}
        \STATEPTCODE{1}{scheduler.step()}
    \end{algorithmic}
    \caption{PyTorch-style pseudocode of generalized training process (Implemented Version)}
    \label{alg-pytorch-train}
\end{algorithm}

\section{D. Optimization of $\boldsymbol{\pi_{\text{DUEL}}}$ Implementation}
\label{apdx_resource_usage}

$\pi_{\text{DUEL}}$ selects the element with the lowest distinctiveness information among the elements in the memory. However, when new data arrives in batches, repeatedly calculating distinctiveness information for each replacement to find the optimal element introduces computational redundancy and reduces efficiency. To address this problem, we replace the calculation of distinctiveness information with a summation over $\qfm{i}{j}$. Since the $-\text{logsumexp}$ function is monotonically decreasing, the replacement calculation using distinctiveness information is equivalent to using a summation over $\qf{i}{j}$. In this case, we call this summation term as the metric of \textit{score}.
\begin{equation*}
J=\argmin{j\in\{1..K\}}\gI_d(x_j;f,\gMp)=\argmax{j\in\{1..K\}}\sum_{k=1}^{K}\qfm{j}{k}.
\end{equation*}

The DUEL process, as a sequential operation that updates an element in a single replacement, is not inherently optimized for GPU computation. To maximize parallelism, we computed $q_{ij}=\qf{i}{j}$ for all pairs within the memory and input batches. We then efficiently computed the \textit{score}s using a selection array indicating the presence of elements in the memory. A visualization of this process can be found in Figure \ref{apdx_d_vis_alg}. If we define the representation as Z-dimensional, the memory size as K, and the batch size as B, the complexity without any optimizations would be $O(ZBK^2)$. However, by using local insertion and deletion, the complexity can be reduced to $O(Z(B+K)^2)$. Algorithm \ref{alg-pytorch-duel} represents the PyTorch-style pseudocode of the DUEL policy that incorporates these optimizations.

\begin{figure}[ht]
\centering
\includegraphics[width=0.8\textwidth]{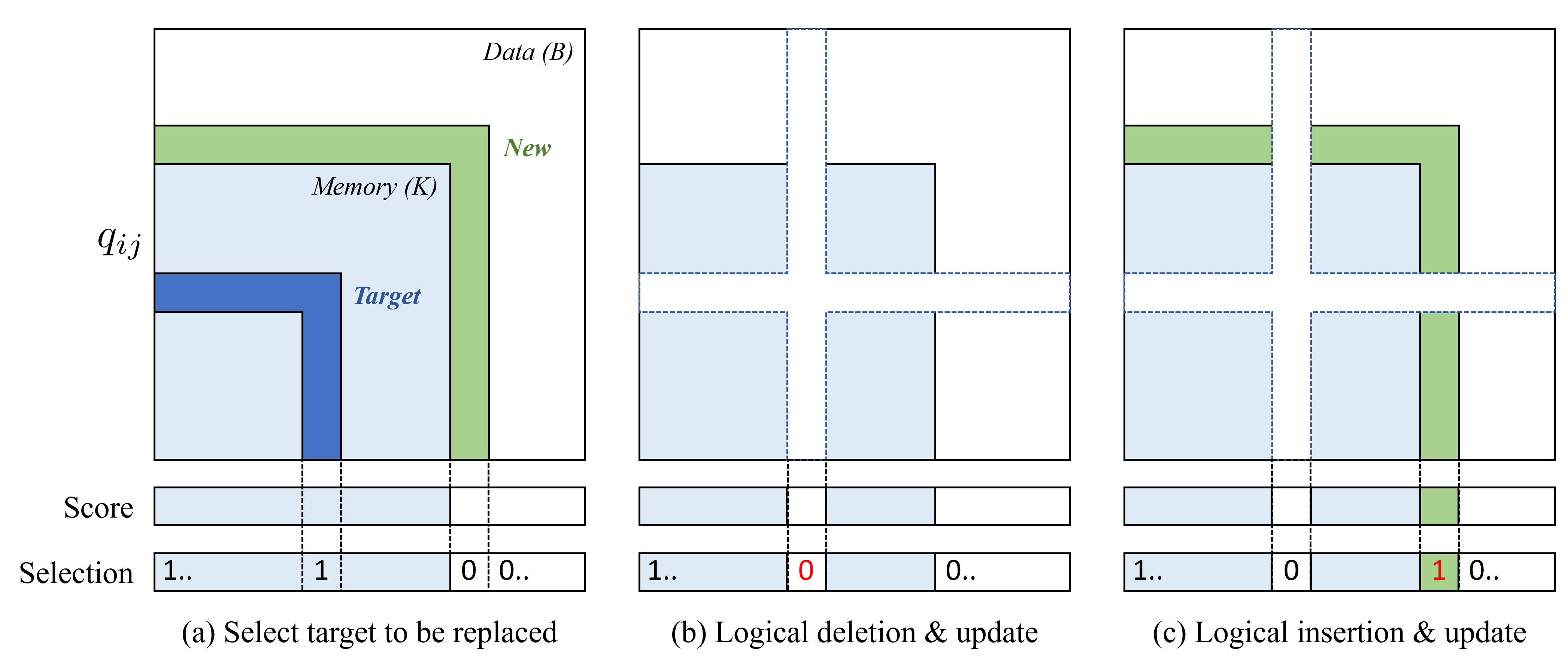}
\caption{Visualization of the algorithm of $\pi_{\text{DUEL}}$. 
By maintaining a selection array, the agent can logically proceed with replacements in the memory, which can increase the efficiency of the process.}
\label{apdx_d_vis_alg}
\end{figure}

\begin{algorithm}[ht]
    \small
    \begin{algorithmic}[1]
        \STATEPTCOMMENT{0}{B: Batch size, K: Memory size}
        \STATEPTCOMMENT{0}{for experiments, score\_function(x) = (1 + x) / 2}
        \STATEPTCODE{0}{def update\_DUEL(memory, p\_target):}
            \STATEPTCODE{1}{comb = torch.cat((memory, p\_target),dim=0)}
            \STATEPTBREAK
            \STATEPTCOMMENT{1}{If memory is not full}
            \STATEPTCODE{1}{if len(memory) < K:}
                \STATEPTCODE{2}{return comb}
            \STATEPTCODE{1}{else:}
                \STATEPTCOMMENT{2}{Build score matrix}
                \STATEPTCODE{2}{score\_matrix = score\_function(torch.mm(comb, comb.transpose(0,1)))}
                \STATEPTBREAK
                \STATEPTCOMMENT{2}{Initialize score and selection array}
                \STATEPTCODE{2}{score\_memory = score\_matrix[:K,:K]}
                \STATEPTCODE{2}{selection = torch.tensor([i < K for i in range(K + B)])}
                \STATEPTCODE{2}{score = torch.cat((score\_memory.sum(-1), torch.zeros(B,)),0)}
                \STATEPTBREAK
                \STATEPTCOMMENT{2}{DUEL policy}
                \STATEPTCODE{2}{for i in range(K, K + B):}
                    \STATEPTCOMMENT{3}{(a) Select target}
                    \STATEPTCODE{3}{J = torch.argmax(score)}
                    \STATEPTBREAK
                    \STATEPTCOMMENT{3}{(b) Logical deletion}
                    \STATEPTCODE{3}{score -= score\_matrix[J] * selection.double()}
                    \STATEPTCODE{3}{selection[J] = False}
                    \STATEPTCODE{3}{score[J] = 0.}
                    \STATEPTBREAK
                    \STATEPTCOMMENT{3}{(c) Logical insertion}
                    \STATEPTCODE{3}{t = score\_matrix[i] * selection.double()}
                    \STATEPTCODE{3}{score += t}
                    \STATEPTCODE{3}{selection[i] = True}
                    \STATEPTCODE{3}{score[i] = t.sum() + \textit{MAX\_SCORE}}
                \STATEPTCODE{2}{return comb[selection]}
    \end{algorithmic}
    \caption{PyTorch-style pseudocode of DUEL policy (Implemented Version)}
    \label{alg-pytorch-duel}
\end{algorithm}

The following table shows the memory and time consumption of the existing frameworks and their respective DUEL frameworks. All experiments are conducted using a single GPU with 24GB of VRAM or more. Elapsed time and VRAM usage are reported using the NVIDIA RTX 3090 TI as a reference. While the DUEL framework requires more time than the existing frameworks, we have observed that additional training does not necessarily lead to improved performance. Further analysis of this observation is discussed in Appendix F.

\begin{table}[ht]
\centering
\renewcommand*{\arraystretch}{1.15}
\begin{tabular}{c|cccc}
\Xhline{3\arrayrulewidth}
& MoCoV2 & SimCLR & D-MoCo & D-SimCLR \\
\hline
Time (h)& 9.37 & 13.47 & 10.99$\scriptstyle(+1.62)$ & 17.34$\scriptstyle(+3.87)$ \\
VRAM (MB) & 7764 & 13196 & 7764$\scriptstyle(+0)$ & 13198$\scriptstyle(+2)$ \\
\Xhline{3\arrayrulewidth}
\end{tabular}
\caption{Resource usage (CIFAR-10, 160k steps with batch size 256, NVIDIA RTX 3090 Ti).}
\end{table}

\section{E. Experimental Detail}
\label{apdx_experimental_detail}
\subsection{E.1. Model Description}
For comparative studies, we design different models based on the MoCoV2 and SimCLR frameworks. The DUEL framework, with its additional memory, is applied to SimCLR. In this case, we retrieve additional negative samples from the memory to complement the batch-wise negative samples used in the original framework. Since MoCoV2 also uses batch-wise negative samples, the structural difference between D-MoCo and D-SimCLR lies in the form of the retrieved data. Figure \ref{apdx_e_vis} provides visualizations for the different implementations.

\begin{figure}[ht]
\centering
\includegraphics[width=0.8\textwidth]{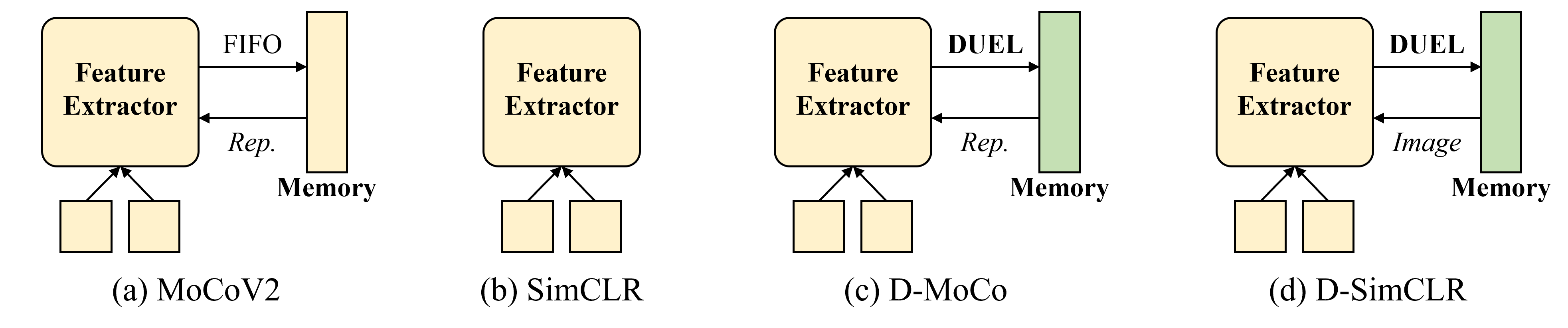}
\caption{Visualization of the DUEL Framework and its origin models used in the experiments.}
\label{apdx_e_vis}
\end{figure}

We apply a stop-gradient operation to the data retrieved from the memory to improve the efficiency of the learning algorithm. In the case of D-SimCLR, the retrieved data is used to extract representations using the current feature extractor. In the implementation, we experiment with setting the value of $\epsilon$ in Equation 17 of general InfoNCE to 1. We also include additional comparative experiments where we set $\epsilon$ to 0, and use only the elements in the memory are used as negative samples. See Appendix F for the quantitative result.

\subsection{E.2. Description of Datasets and Hyperparameters}

\paragraph{CIFAR-10}{
In the experiments with CIFAR-10, we use ResNet-50 as the backbone for the feature extractor. To address the issue of image size, we follow the approach of Chen et al.~\citep{chen2020simple} and reduce the kernel size of the first CNN layer in the resnet. The resnet features are projected into a 256-dimensional space using a single projection layer and then normalized. For all models, we set the value of $\tau$ to 0.5, and for the MoCo variants, we set the momentum factor to 0.9. The value of $\lambda$ for Barlow Twins is set to 5e-3, as suggested in the original paper. The learning rate starts at 0.05 and decays using a cosine scheduler. The adam optimizer~\citep{kingma2014adam} is used for training. To investigate the effect of the number of training steps on performance, we train for 160k steps. In all experiments, the memory size is set to 2048. For D-SimCLR, 256 samples were randomly extracted from memory and used for additional negative samples. For BYOL, two DNN layers are used as projection layers and the dimension of the hidden layer is set to 2048.
}

\paragraph{STL-10}{
For STL-10, we also use ResNet-50 as the backbone, and no specific modifications are made due to the larger image size. Since only downstream performance is evaluated in this experiment, a 2-layer projection layer with ReLU activation is used. The images are embedded in a 512-dimensional space by the projection layer. We train for 40k steps. The hyperparameter settings are the same as in CIFAR-10, except for those mentioned above. For BYOL, two DNN layers are used as projection layers and the dimension of the hidden layer is set to 4096. And we apply batch normalization to the intermediate layers of BYOL.
}

\paragraph{ImageNet-LT}{
ImageNet-LT~\citep{liu2019large}, which consists of a total of 115.8k images with 1000 classes with a long-tailed class distribution: the most frequent class has 1280 images and the least frequent class has only 5 images. We resize the image to 64$\times$64. We also use the ResNet-50 as the backbone, with a projection layer of three 512-dimensional DNN layers. ReLU is used as the activation function. The model is trained with a batch size of 256 for 160k steps. Other settings are the same as in the STL-10 experiments. For downstream tasks, we use the validation set of the original ImageNet dataset and Tiny-ImageNet~\citep{le2015tiny}, which is another modified subset of ImageNet. We also resize the images to fit the model.
}

\paragraph{Augmetation}{
For augmentation, we employ color distortion and horizontal flip as suggested by Chen et al~\citep{chen2020simple}. Color distortion is applied with a probability of 0.5, while horizontal flip is applied with a probability of 0.75. Color distortion includes color jittering and grayscale.
}

\paragraph{Linear Probing}{
The trained CNN layers of the model were frozen, and only one linear layer was added to perform the classification task. The training was carried out using the Adam optimizer with a weight decay of 1e-6. The learning rate was set to 0.001 and the batch size was 256, and the training was performed for a total of 100 epochs. For ImageNet-LT, an additional learning rate decay was applied to reduce the learning rate by half every 10 epochs.
}

\clearpage

\section{F. Experimental Result}
\label{apdx_experimental_result}

This section presents additional experimental settings or further analysis of experimental results not covered in the main experiment. In the case of CIFAR-10, a larger number of training steps was used to observe the effect of extended training on the model. The performance of the model was also evaluated on the STL-10 and ImageNet-LT. In addition, a more detailed analysis of the mechanisms behind the DUEL policy is included.

\subsection{F.1. The Effect of Extended Training}

To examine the performance changes in our DUEL framework, we perform evaluations at checkpoints during the training process. Figure \ref{linear_probing_step} shows the plotted model performance over time for the proposed implementations and their original frameworks. It is evident that the proposed framework consistently outperforms the existing frameworks. In particular, in the case of SimCLR, we observe a drop in performance after the mid-training phase, which can be analyzed as an effect of overfitting. This suggests that additional training may have a negative impact on the performance of the model.

\begin{figure}[ht]
    \begin{center}
    \begin{subfigure}[b]{0.34\textwidth}
    \includegraphics[width=\textwidth]{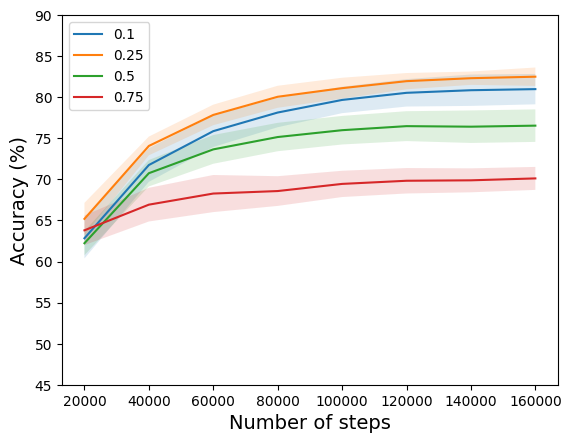}
    \caption{Linear probing accuracy (MoCo, \%).}
    \end{subfigure}%
    \hspace{0.5em}
    \begin{subfigure}[b]{0.34\textwidth}
    \includegraphics[width=\textwidth]{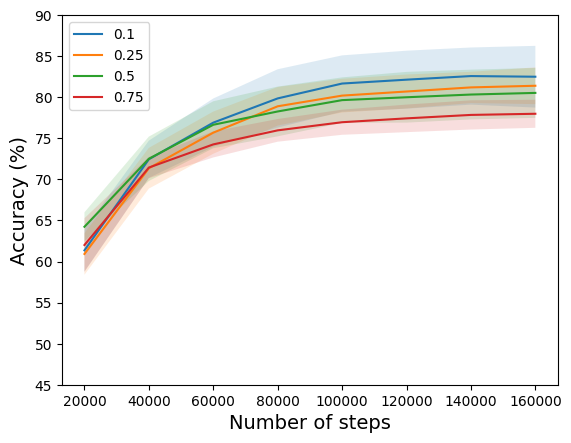}
    \caption{Linear probing accuracy (D-MoCo, \%).}
    \end{subfigure}
    \begin{subfigure}[b]{0.34\textwidth}
    \includegraphics[width=\textwidth]{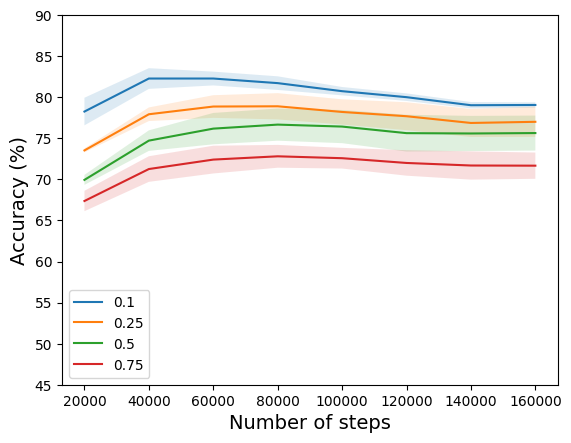}
    \caption{Linear probing accuracy (SimCLR, \%).}
    \end{subfigure}%
    \hspace{0.5em}
    \begin{subfigure}[b]{0.34\textwidth}
    \includegraphics[width=\textwidth]{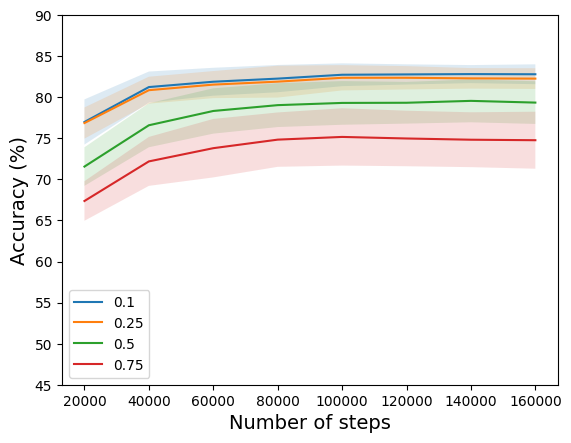}
    \caption{Linear probing accuracy (D-SimCLR, \%).}
    \end{subfigure}
    \end{center}
    \caption{Visualization of the change in linear probing accuracy along training steps in the comparative models. The translucent interval in each graph represents the standard deviation.}
    \label{linear_probing_step}
\end{figure}

We observe that our framework has a relatively large standard deviation in performance, and we also observe that the arrangement of representations in the early stages of training can affect subsequent performance. This phenomenon can be attributed to the use of an uncertain feature extractor in the early stages. Further discussion is needed to explore potential approaches to mitigate this problem.

\clearpage
\subsection{F.2. Additional Analyses on Active Memory and Latent Space}

The DUEL policy involves replacing data from the dominant class, thereby increasing the entropy of the class distribution. Figure \ref{class_entropy} illustrates the entropy changes in the class distribution of items in memory during the replacement process. We can observe a gradual increase in entropy from the initialized state with $n=|\gM|$ (denoted as 1M) to the final state with $n=5\cdot|\gM|$ (denoted as 5M). We also compare the properties of the learned features between MoCo and D-MoCo by calculating the intra-class variance and the inter-class similarity for each class. Figure \ref{intra-inter} visualizes the results.

\begin{figure}[ht]
    \begin{center}
    \begin{subfigure}[ht]{0.28\textwidth}
    \includegraphics[width=\textwidth]{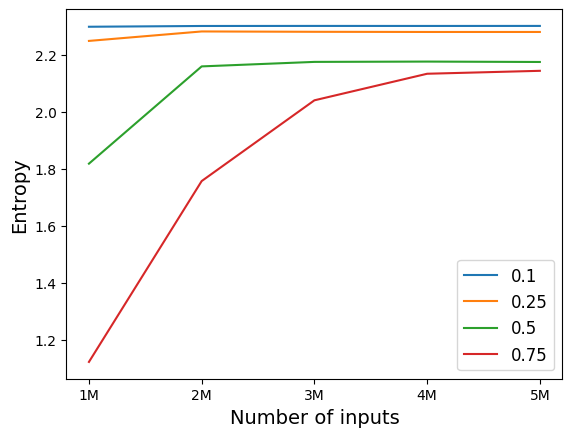}
    \caption{Class entropy}
    \label{class_entropy}
    \end{subfigure}%
    \begin{subfigure}[ht]{0.3\textwidth}
    \includegraphics[width=\textwidth]{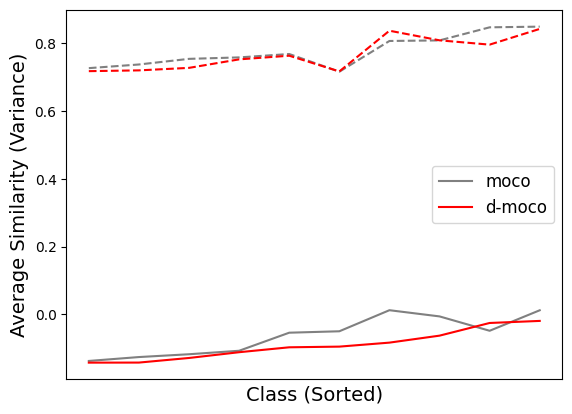}
    \caption{Class-wise metrics ($\rho_{\max}=0.5$)}
    \label{rep-05}
    \end{subfigure}%
    \begin{subfigure}[ht]{0.3\textwidth}
    \includegraphics[width=\textwidth]{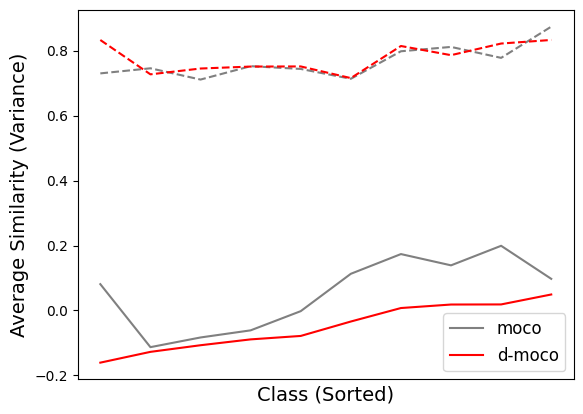}
    \caption{Class-wise metrics ($\rho_{\max}=0.75$)}
    \label{rep-075}
    \end{subfigure}%
    \end{center}
    \caption{This visualization shows the intra-class variance (dotted) and the inter-class similarity (solid) for each class. The classes are rearranged for the convenience of comparison.}
    \label{intra-inter}
\end{figure}

In Figure \ref{rep-05} and Figure \ref{rep-075}, the top two plots represent the intra-class variance, while the bottom two plots represent the inter-class similarity. Overall, the intra-class variance values show similar trends between the two methods. This indicates that the DUEL policy does not significantly affect the \textit{concentration} of the learned representations. However, in terms of inter-class similarity, D-MoCo consistently yielded lower values compared to MoCo in almost all cases. This suggests that D-MoCo effectively extracts representations by ensuring sufficient separation between the clusters that form each class.

The figure \ref{class_frequency} shows the frequencies of the classes stored in memory, sorted in descending order. It can be seen that the information of the dominant class, shown on the far left, is progressively suppressed as data is entered.

\begin{figure}[ht]
    \begin{center}
    \begin{subfigure}[b]{0.33\textwidth}
    \includegraphics[width=\textwidth]{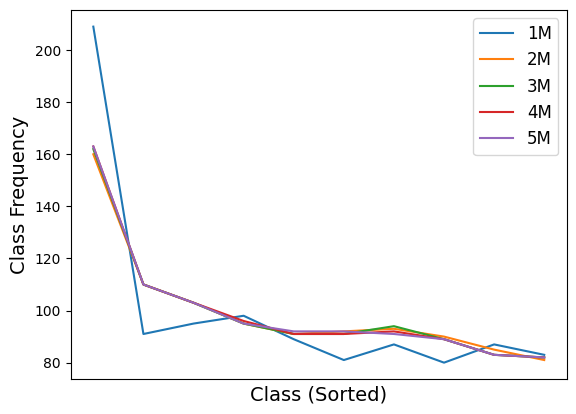}
    \caption{Class frequency ($\rho_{\max}=0.25$).}
    \end{subfigure}%
    \begin{subfigure}[b]{0.33\textwidth}
    \includegraphics[width=\textwidth]{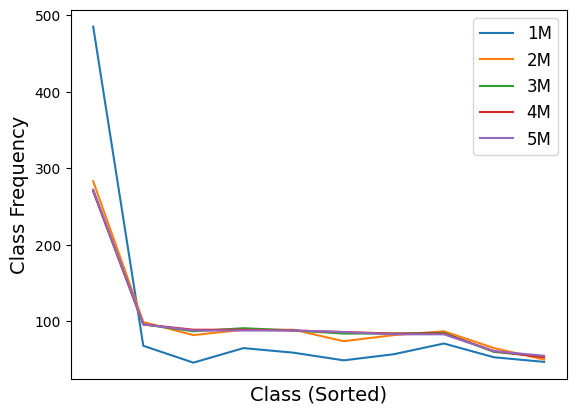}
    \caption{Class frequency ($\rho_{\max}=0.5$).}
    \end{subfigure}%
    \begin{subfigure}[b]{0.33\textwidth}
    \includegraphics[width=\textwidth]{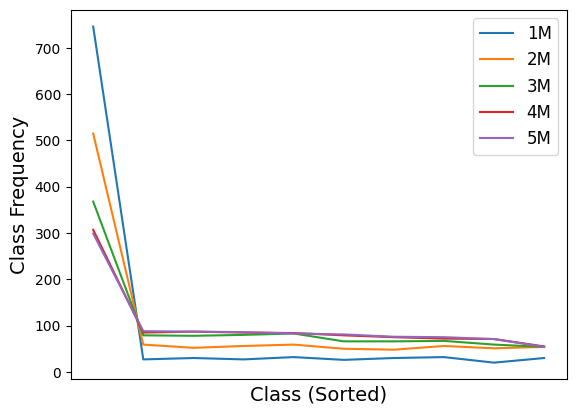}
    \caption{Class frequency ($\rho_{\max}=0.75$).}
    \end{subfigure}
    \end{center}
    \caption{Visualization of class frequency shifts in memory in various class-imbalanced settings. The classes are rearranged for the convenience of comparison.}
    \label{class_frequency}
\end{figure}

\clearpage
\subsection{F.3. Ablation Study: Choosing Negative Samples for D-SimCLR}

The SimCLR model used different data from the same minibatch as negative samples. However, due to the DUEL framework's additional use of elements within memory, we perform an ablation study of different methods for selecting negative samples. To do this, we experiment by varying two criteria during the study: epsilon, which indicates whether the numerator term is added to the denominator term of the InfoNCE loss, and the methods for obtaining negative samples. For the negative sample acquisition methods, we use three approaches: "batch only", "memory only", and a combination of the two, referred to as "mixed". To minimize the cost incurred by gradient operations on data within the memory, we applied a stop-gradient operation to features extracted from the data, preventing additional computations.

\begin{table}[ht]
    \renewcommand*{\arraystretch}{1.25}
    \begin{center}
        \begin{adjustbox}{width=0.65\textwidth}
            \begin{tabular}{ccc|cccc}
            \Xhline{3\arrayrulewidth}
            \multicolumn{3}{c|}{\bf Model Structure}&\multicolumn{4}{c}{\bf Class Probability $\rho_{\max}{\scriptstyle (\rho_{\min})}$}\\
            Structure&Negatives&$\epsilon$&0.1 {\small(0.1)}&0.25 {\small(0.083)}&0.5 {\small(0.056)}&0.75 {\small(0.028)}\\
            \hline
            SimCLR&batch only&1&$82.28{\scriptstyle\pm1.26}$&$78.90{\scriptstyle\pm1.60}$&$76.67{\scriptstyle\pm1.96}$&$72.81{\scriptstyle\pm1.39}$\\
            \hline
            \multirow{3}{*}{D-SimCLR}&memory only&0&$71.26{\scriptstyle\pm1.77}$&$71.10{\scriptstyle\pm0.28}$&$67.42{\scriptstyle\pm1.65}$&$60.27{\scriptstyle\pm1.07}$\\
            &mixed&0&$81.00{\scriptstyle\pm1.09}$&$80.43{\scriptstyle\pm2.08}$&$77.93{\scriptstyle\pm1.18}$&$74.85{\scriptstyle\pm3.25}$\\
            &mixed&1&${\bf82.82}{\scriptstyle\pm1.10}$&${\bf82.37}{\scriptstyle\pm1.53}$&${\bf79.56}{\scriptstyle\pm2.61}$&${\bf75.17}{\scriptstyle\pm3.49}$\\
            \Xhline{3\arrayrulewidth}
            \end{tabular}
        \end{adjustbox}
    \end{center}
    \caption{Downstream task accuracies for ablation study with various model structures. (CIFAR-10, 3 times, \%, Linear Probing (Top-1))}
    \label{ablation}
\end{table}

Table \ref{ablation} shows the results for different maximum class probabilities on the CIFAR-10 dataset. It confirms that the best performance is achieved when both data from the same batch and data sampled from memory are used as negatives. Therefore, we use this optimal setting consistently in other experiments.

\subsection{F.4. Quantitative Result: ImageNet-LT}

To evaluate the applicability of the DUEL framework in more realistic settings, we conducted additional experiments in a more complicated environment. The ImageNet-LT dataset defines the class distribution as a long-tailed distribution to simulate real-world environments. Moreover, since ImageNet is a dataset for a classification task with 1000 classes, experiments on this dataset provide sufficient validation for the scalability of our model. For experimental verification, we evaluate the linear probing performance on several datasets with different tasks.

\begin{table}[ht]
    \renewcommand*{\arraystretch}{1.25}
    \begin{center}
        \begin{adjustbox}{width=0.6\columnwidth}
            \begin{tabular}{c|cccc}
            \Xhline{3\arrayrulewidth}
            Datasets&MoCo&D-MoCo&SimCLR&D-SimCLR\\
            \hline
            Tiny-ImageNet (train)&42.61$\scriptstyle(\pm0.83)$&\bf46.64$\scriptstyle(\pm0.60)$&63.17$\scriptstyle(\pm1.63)$&\bf65.84$\scriptstyle(\pm0.28)$\\
            STL-10 (test)&82.47$\scriptstyle(\pm0.31)$&\bf86.19$\scriptstyle(\pm1.43)$&98.90$\scriptstyle(\pm0.73)$&\bf99.31$\scriptstyle(\pm0.01)$\\
            ImageNet (val)&86.59$\scriptstyle(\pm0.20)$&\bf88.59$\scriptstyle(\pm1.65)$&99.94&99.94\\
            \Xhline{3\arrayrulewidth}
            \end{tabular}
        \end{adjustbox}
    \end{center}
    \caption{Downstream task accuracies via linear probing procotol. (ImageNet-LT, 3 times, \%, Linear Probing (Top-1))}
    \label{expimagenetlt}
\end{table}

Table \ref{expimagenetlt} compares the performance of the proposed DUEL framework with that of the original models on these datasets. In most cases, the DUEL framework shows similar or better performance than the original models. This implies that the representations extracted by the DUEL framework are more robust for general tasks.

\end{document}